\documentclass{comsoc2020}

\usepackage{amsmath}
\usepackage{array}
\usepackage{amssymb}
\usepackage{amsthm}
\usepackage{cleveref}
\usepackage{thmtools}
\usepackage{booktabs}
\usepackage{graphicx}
\usepackage{tikz}
\usepackage{comment}
\usepackage{pgfplots}
\usepackage{array}
\usepackage{subcaption}
\usepackage{paralist}
\usepackage{natbib}

\usepackage[ruled,vlined,linesnumbered]{algorithm2e}

\usepackage{bbm}
\usepackage{tikz}
\usepackage{wrapfig}
\usetikzlibrary{backgrounds}
\usetikzlibrary{positioning}
\usetikzlibrary{shapes}

\def\compileFigures{0}
\newcommand{\filename}{main}
\newcounter{figuerNumber}
\if\compileFigures1
\usetikzlibrary{external}
\fi

\usepackage{pgfplots}
\usepackage{pgfplotstable}
\pgfplotsset{compat=1.16}
\usepgfplotslibrary{colormaps}
\pgfplotsset{colormap/gray}
\usepackage{ifthen}

\tikzset{cellColor/.style={color of colormap={#1},fill=.}}
\newcommand{\colorcell}[3][black]{\begin{tikzpicture}[inner sep=0pt, baseline=0, show background rectangle]
      \pgfmathparse{int((#2+#3)/(2*#3)*1000)}\let\bkvalue\pgfmathresult
      \tikzset{background rectangle/.style={cellColor=\bkvalue}};
        \ifnum \bkvalue < 750
            \ifnum \bkvalue > 250
                \draw (0,0) node[text=black,minimum width=19pt] {#2};
            \else
                \draw (0,0) node[text=white,minimum width=19pt] {#2};
            \fi
        \else
            \draw (0,0) node[text=white,minimum width=19pt] {#2};
        \fi
      
    \end{tikzpicture} }
    
\newcommand{\bc}[1]{\begin{tikzpicture}[inner sep=0pt, baseline=0, show background rectangle]
      \pgfmathparse{int(1000)}\let\bkvalue\pgfmathresult
      \tikzset{background rectangle/.style={cellColor=\bkvalue}};
      \draw (0,0) node[text=white,minimum width=19pt] {#1};
    \end{tikzpicture} }
\newcommand{\whitecell}[1]{\begin{tikzpicture}[inner sep=0pt, baseline=0,background rectangle/.style={fill=white}, show background rectangle]
      \draw (0,0) node[minimum width=19pt] {#1};
    \end{tikzpicture} }
\newcommand{\wc}[1]{\whitecell{#1}}
\newcommand{\CGscale}{0.015}
\newcommand{\MSscale}{0.161}

\newcommand{\R}{\mathbb{R}}

\newcommand{\Age}{\texttt{Age}}
\newcommand{\CG}{\texttt{Capital Gain}}
\newcommand{\MS}{\texttt{Marital Status}}
\newcommand{\EL}{\texttt{Education Level}}
\newcommand{\EN}{\texttt{Education Number}}

\newcommand{\poi}{\vec x}

\newcommand{ \BII}{\texttt{BII }}
\newcommand{ \SII}{\texttt{SII }}
\newcommand{ \QII}{\texttt{QII }}
\newcommand{ \LIME}{\texttt{LIME }}

\DeclareMathOperator*{\argmin}{arg\,min}
\usepackage{amsthm}
\newtheorem{example}{Example}
\newtheorem{theorem}{Theorem}

\usepackage{blindtext}

\usepackage{etoolbox}

\title{High Dimensional Model Explanations: An Axiomatic Approach}
\author{Neel Patel, Martin Strobel and Yair Zick}

\begin{document}

\begin{abstract}
Complex black-box machine learning models are regularly used in critical decision-making domains.  
This has given rise to several calls for algorithmic explainability. Many explanation algorithms proposed in literature assign importance to each feature individually. However, such explanations fail to capture the joint effects of sets of features. Indeed, few works so far formally analyze \emph{high dimensional model explanations}. In this paper, we propose a novel high dimension model explanation method that captures the joint effect of feature subsets.
 
We propose a new axiomatization for a generalization of the Banzhaf
index; our method can also be thought of as an
approximation of a black-box model by a higher-order polynomial. In other words, this work justifies the use of the generalized Banzhaf index as a model explanation by showing that it uniquely satisfies a set of natural desiderata and that it is the optimal local approximation of a black-box model. 
 
Our empirical evaluation of our measure highlights how it manages to capture desirable behavior, whereas other measures that do not satisfy our axioms behave in an unpredictable manner.  
\end{abstract}

\section{Introduction}\label{sec:intro}
Machine learning models are applied in a variety of high-stakes domains, such as healthcare, insurance, credit decisions and more. 
In order to offer high prediction accuracy over high-dimensional data, one must adopt increasingly complex models. Complex models are, however, more difficult to understand.
Explainability has been recently identified as an important design criterion, and has received widespread attention within the ML community; of particular note is research into generating \emph{model explanations}. 
Broadly speaking, model explanations are based on a labeled dataset as well as other potential sources of information. Most model explanation techniques focus on feature-based explanations: for each feature $i$, the model explanation outputs a value $\phi_i$ which signifies the importance of the $i$-th feature in determining model predictions.  
 
The value $\phi_i$ can be thought of as a score --- `Alice's high income is highly indicative of her receiving a loan' --- or as a recourse --- `had Alice's income been lower by \$10,000/year her loan would have been rejected'.
Either way, the basic premise of attribute-based model explanations is to explain complex model decisions via a list of $n$ numerical values, where $n$ is the number of data features. 
Crucially, this approach fails to capture \emph{feature interactions}. 
Features are often strongly interdependent, especially in \emph{complex ML models}.  For example: consider a black-box model that predicts sentiments associated with a paragraph of text. In such texts; there can be a high negative interaction effect between ``not'' and ``bad'', which attribute-based model explanations will fail to capture: assigning influence to ``bad'' and ``not'' individually can be misleading. 
Consider the following review:

\begin{quote}
    \texttt{It isn't the greatest scifi flick I've every seen, but it is not a bad movie}
\end{quote}
The above review is classified as a positive sentence. \LIME \cite{ribeiro2016should} --- a type of feature-based explanation --- assigns high negative influence to the words ``bad'' and ``not''. However, reviews omitting the word ``bad'' or ``not'' will be classified as negative. Explanation methods that indicate feature interactions, rather than just standalone scores, offer a much clearer picture of how the model makes its decisions. 
 
The current model explanation landscape offers a wide variety of model explanation methods; how should we pick the right one? In this paper, we propose an \emph{axiomatic approach}: rather than starting with a candidate solution and then offering post-hoc justifications for using it, start with the properties one would like to have in a high-dimensional model explanation, and derive a solution that satisfies these properties. This is the approach we take in this paper.

We axiomatize a high dimensional model explanation method, called the Banzhaf Interaction Index, or BII. This method faithfully captures how feature interactions influence model decision-making. Not only does our proposed model explanation satisfy a set of desirable properties, it is the \emph{only} model explanation that satisfies all these properties. 

\subsection{Our Contributions}
In this work, we propose a method for high-dimensional explanations for black-box models. Our main goal is to axiomatically capture higher-order feature interaction. 
Our main contribution is twofold: first, we propose a high-dimensional black-box model explanation method, axiomatically capturing higher-order feature interaction. Our characterization uses properties that are very natural in the ML context, such as feature symmetry and effect monotonicity (Section \ref{sec:axiom}). we obtain a new axiomatization of the \textit{Banzhaf interaction index} which uniquely satisfies symmetry, limit condition, general-2-efficiency, and a newly proposed axiom (in the context of Banzhaf indices): monotonicity. Monotonicity is a rather general property which essentially means that the model explanation should change in a manner faithful to the underlying data. This is very fundamental property for an interaction measure: which states that the net contribution of the subset of features for the machine learning model $f$ is more than that for the model $g$; then the interaction measure for those features for model $f$ should be more than the interaction measure for those features for model $g$.

Second, we extend the idea of feature-based model explanations, which can be thought of as a local linear approximations of black-box models, to higher-order polynomial approximations. Especially, we show that our proposed measure can be obtained by approximating the black-box model by a higher-order polynomial (Section~\ref{sec:geo}).  

We evaluate our approach on real datasets. In particular, our experiments show that feature-based model explanations may fail to capture feature synergy; moreover, we show that other model explanations which fail to satisfy our axioms tend to behave in an undesirable manner (Section \ref{sec:experiments}). 

\subsection{Related Work}
While some model explanations provide record-based explanations \cite{koh2017understanding}, or generate explanations from source code \cite{datta2017use}, the bulk of the literature on the topic focuses on feature-based model explanations \cite{ancona2017unified,baehrens2010explain,datta2015influence,datta2016algorithmic,sliwinski2019axiomatic,sundararajan2017axiomatic}. \citet{ancona2018towards} offer an overview of feature based model explanations for deep neural networks. 
The connection to \emph{cooperative game theory} has been widely discussed and exploited in order to generate model explanations \cite{ancona2017unified,datta2015influence, datta2016algorithmic}, with a particular focus on the Shapley value and its variants \cite{shapley1953value}.   
\paragraph{Interaction Indices for Cooperative Games:} Two widely accepted measures of marginal influence from cooperative game theory the Shapley value \cite{shapley1953value} and the Banzhaf value \cite{banzhaf1964weighted}, are uniquely derived from a set of natural axioms (see also \cite{young1985monotonic}).
These measures do not capture player interactions; rather, they assign weights to individual players. \citet{owen1972multilinear} proposes the first higher-order solution for a cooperative game for pairwise players. 
\citet{grabisch1999axiomatic} extend it to interaction between arbitrary subsets, and axiomatically derive the Shapley and the Banzhaf interaction indices. 
In a recent paper, \citet{agarwal2019new} propose a new axiomatization for interaction among players which is inspired by the Taylor approximation of a Boolean function.  
\paragraph{Feature Interactions:} Feature interaction has been studied by different communities. Statistics offers classic ANOVA based feature interaction analysis \cite{fisher1992statistical,gelman2005analysis}. 
Some recent work in the deep learning literature discusses feature interaction: \citet{tsang2018neural} learns interactions by inspecting the inter-layer weight matrices of a neural network, \citet{tsang2017detecting} construct a generalized
additive model that contains feature interaction information. 
In another line of work, \citet{greenside2018discovering, cui2019recovering} compute the interaction among features by computing the (expected) Hessian; this can be thought as an extension of gradient-based influence measures for neural networks \cite{ancona2018towards}. \citet{datta2016algorithmic} also propose an influence measure for a set of features called Set-QII. It essentially measures the change in model output when we randomly change a fixed set of features. 
\citet{lundberg2018consistent} propose the Shapley interaction index, specifically for tree-based models. 
\section{Preliminaries}\label{sec:prelim}
We denote sets with capital letters $A,B,\dots$ and use lowercase letters for functions, scalars, and features. 
To minimize notation clutter, we try to omit braces for singletons, pairs etc., e.g. we write $f(i) , S \cup i$ instead of $f(\{i\}) , S \cup \{i\}$ and   $S \cup ij $ instead of $S \cup \{i,j\} $.
 
Let $N=\{1,\dots,n\}$ be the set of \emph{features}. 
A \emph{black-box model} is a function mapping a set of $n$-dimensional input vectors $\mathcal X \subseteq \R^n$ to $\R$. 
Our objective is to generate a \emph{model explanation} for a given \emph{point of interest} (POI) $\poi \in \R^n$; 
this explanation should (ideally) offer stakeholders some insight into the underlying decision-making process that ultimately resulted in the outcome they receive. 
In this work, we are interested in measuring the extent to which features, and \emph{their high-order interactions} affect model decisions.
In order to measure feature interaction effects, we adopt the \emph{baseline comparison} approach \cite{lundberg2017unified,sundararajan2019many}; in other words, we assume the existence of a \emph{baseline vector} $\vec x'$, to which we compare an input vector $\poi$, in order to generate a model explanation. 
For example, in the automatic loan acceptance/rejection domain $\vec x'$ could correspond to an all-zero vector (e.g. measuring the effect of an applicant having no money in their bank account, as opposed to the true amount they have), or a vector of mean values (e.g. comparing the applicant's true income to the average population income). 
 
In order to generate model explanations, we need to formally reason about the effect of changing features in the POI $\poi$ to their baseline values. Generally speaking, changing a single feature may have no significant effect on the model prediction. For example, if $f(\mathit{income}=20k,\mathit{debt}=90k) = 1$, it may well be the case that changing either the applicant's low income ($20k$) or their high debt ($90k$) would not result in them receiving the loan, however, it is unreasonable to claim that neither had an effect on the outcome. 
To formally reason about the joint effect of features, we define a function measuring their value as a set. Given a set of features $S\subseteq N$, a point of interest $\poi$, a baseline $\vec x'$ and a model $f$, we define a set function as 
\begin{align}
v(S,\poi, \vec x',f) = f(\poi_S,\vec{x}'_{N\setminus S}) - f(\vec x').\label{eqn:importance-game}
\end{align}
In other words, the value we assign to a set of features $S$ is the extent to which they cause the model prediction to deviate from the baseline prediction; as a sanity check, we note that $v(\emptyset,\poi,\vec x',f) = 0$, and $v(N,\poi,\vec x',f) = f(\vec x) - f(\vec x')$. 
This formulation induces a \emph{cooperative game}, where features correspond to players. We refer to the game defined in \eqref{eqn:importance-game} as the \emph{feature effect game}. 
We omit $\vec x$, $\vec x'$ and $f$ (they will normally be fixed), focusing solely on the set of features $S$. 
 
We often replace sets of features with a single feature demarcating the entire set: given a set $T\subseteq N$, $[T]$ denotes a \emph{single} feature corresponding to the set. 
The \emph{reduced game} w.r.t. the nonempty $T\subseteq N$ is defined on the features $N \setminus T\cup [T]$ with the characteristic function $v_{[T]}(\cdot,f):2^{N\setminus T \cup\{[T]\}} \to \R$:
\begin{align*}
v_{[T]}(S) =& \begin{cases}
   v(S\cup T)    & [T] \in S   \\
   v(S)           & \text{otherwise} 
\end{cases}
\end{align*}
Our objective is to generate \emph{high dimensional model explanations}, i.e. functions that assign a value to every subset of features $S \subseteq N$. 
To do so, we define a \emph{feature interaction index} $S\subseteq N$ for $v$ as $I^{v}(S)$. 
In other words, under the game $v$ --- namely the game defined in \eqref{eqn:importance-game} --- $I^v(S)$ should roughly capture the overall effect that the set of features $S$ has on the value of $v$. 
Going back to the feature effect game defined in \eqref{eqn:importance-game}, $I^v(S)$ should measure the degree to which switching the value of features in $S$ back to their baseline values affects the prediction for $\vec x$. 
A key idea in our analysis is \emph{marginal effect}: consider a single feature $i \in N$. Its marginal effect on a set $T \subseteq N \setminus \{i\}$ equals $v(T\cup i) - v(T)$: the extent to which the value of $v(T)$ changes when $i$ joins $T$. 
The marginal effect of $i$ on $T$ is denoted $m_i(T)$. 
In the context of the feature importance game, $m_i(T,v)$ is the marginal effect of knowing the feature $i$, given that the values of features in the set $T$ are known. This is similar to other definitions considered in the literature \cite{datta2016algorithmic,lundberg2017unified}.
When considering a pair of features $i,j \in N$, how would one define their marginal effect on a coalition $T \subseteq N$? 
One very natural definition is to offset the effect of adding both features to $T$ by the marginal effects of adding $i$ and $j$ separately, i.e.
\begin{align*}
    m_{ij}(T,v) =& \left(v(T\cup ij) - v(T)\right) -\left(v(T\cup i) -v(T)\right) - \left(v(T\cup j) - v(T)\right)\\
                =& v(T\cup ij) - v(T\cup i) - v(T\cup j) + v(T)
\end{align*}
We can define the marginal contribution of a general $S\subseteq N$, in a similar manner. Let $m_S(T,v)$ be:
\[
    m_S(T,v) = \sum_{L\subseteq S} (-1)^{|S|-|L|} v(T \cup L).
\]
This is also known as the \emph{$S$ discrete derivative of $v$} at $T$ (echoing the inclusion-exclusion principle).
We note that when $T = \emptyset$, $m_S(\emptyset,v)$ is the \emph{Harsanyi dividend} of $S$, a well-known measure of the synergy (or surplus) generated by $S$ \cite{coopbook}. For $T \ne \emptyset$, $m_S(T,v)$ can be thought of as the added value of having the coalition $S$ form, given that the players in $T$ have already committed to joining. More generally, $m_S(T,v)$ represents the marginal interaction between features in $S$ within the set of features $T\cup S$. 

\section{Axiomatic Characterization of Good Model Explanations}\label{sec:axiom}
As previously discussed, our objective is to identify \emph{high-quality} model explanations. When pursuing quality metrics for a model explanation, one can take one of two approaches: either show that the model explanation is the optimal solution to some target (e.g. minimizes some loss function) \cite{ribeiro2016should,sliwinski2019axiomatic}, or that it satisfies a set of desirable properties \cite{datta2015influence,datta2016algorithmic,sliwinski2019axiomatic}. 
We take both approaches in this work, resulting in a \emph{unique} and \emph{optimal} high-dimensional explanation; this is the \emph{Banzhaf interaction index} (BII) \cite{grabisch1999axiomatic}. 
Given a cooperative game $v$, the Banzhaf interaction index for a subset $S \subseteq N$ is
\begin{align}
    I^v_{\BII}(S) &= \frac{1}{2^{n-|S|}} \sum_{T\subseteq N \setminus S}\sum_{L \subseteq S} (-1)^{|S|-|L|} v({L \cup T}) \notag \\
    &=\frac{1}{2^{n - |S|}} \sum_{T \subseteq N\setminus S}m_S(T,v)\label{eqn:banzhaf_interaction_index}
\end{align}
In words, $I_{\BII}^v(S)$ equals $S$'s expected marginal contribution to a set $T \subseteq N \setminus S$, sampled uniformly at random. 
\citet{grabisch1999axiomatic} show that \BII uniquely satisfies \emph{Linearity}, \emph{Symmetry}, \emph{Dummy}, the \emph{Recursive Property}, \emph{Generalized 2-Efficiency}, and the \emph{Limit Condition}. 
Theorem \ref{thm:main} proposes a `leaner' axiomatization of \BII, which, as we argue, is more sensible in the model explanation setting.
We show that \BII is the unique measure which satisfies four natural axioms: \emph{Symmetry}, \emph{Generalized 2-Efficiency}, the \emph{Limit Condition}, and \emph{Monotonicity}. 
The first three axioms are fairly standard assumptions in identifying `good' solutions, generalized to interaction indices by \cite{grabisch1999axiomatic}. 
\begin{description}
\item[Symmetry (S):] for any permutation $\pi$ over $N$ we have that $I^v(S) = I^{\pi v}(\pi S)$. Here, $\pi S$ equals $\{\pi(i): i \in S\}$, and $\pi v$ is the game where the value of a coalition $T$ equals $v(\pi^{-1}T)$.
Symmetry is a natural property for any interaction measure: intuitively, it simply stipulates that features' interaction value is independent of their identity, and depends only on their intrinsic coalitional worth. 
\item[Generalized 2-Efficiency (GE):] for any $i,j \in N$, and for any $S \subseteq N \setminus ij$:
\begin{align*}
    I^{v_{[ij]}}(S\cup [ij]) = I^{v}(S\cup i) + I^{v}(S\cup j).\notag
\end{align*}
Intuitively, merging two features into one feature encoding the same information results in no additional influence. 
Generalized 2-Efficiency extends the 2-Efficiency axiom proposed by \cite{lehrer1988axiomatization} to characterize the Banzhaf value.    
\item[Limit Condition (L):] if $N$ is the set of players of the game $v$ then $I^{v}(N) = m_N(\emptyset,v) = \sum_{L \subseteq N}(-1)^{n-|L|}v(L)$. In other words, the interaction value of the set $N$ equals exactly the added value of it forming, given that no subsets of players have pre-committed themselves to joining. 
\end{description}
Next, let us introduce the notion of monotonicity for interaction indices. 
\begin{description}
\item[Monotonicity (M): ] If $\forall T\subseteq N\setminus S$, $m_S(T,v_1) \geq m_S(T,v_2)$ and for some $T^*\subseteq N\setminus S$ strict inequality holds then $I^{v_1}(S) > I^{v_2}(S)$. Moreover, if $\forall T\subseteq N\setminus S$, $m_S(T,v_1) = m_S(T,v_2)$ then $I^{v_1}(S) = I^{v_2}(S)$.
This idea extends the strong monotonicity axiom proposed by \cite{young1985monotonic}, which states that if $m_i(T,v_1)\geq m_i(T,v_2)$ for all $T\subseteq N\setminus i$ then $I^{v_1}(i)\geq I^{v_2}(i)$.
\end{description}

\citet{datta2016algorithmic} argue that the monotonicity axiom is better suited for charactering model explanations than the more `standard' linearity axiom used in the classic characterization of the Shapley value \cite{shapley1953value}, as well as the original \BII characterization \cite{grabisch1999axiomatic}. 

\subsection{Characterization Result}
First, we show that the four axioms we propose imply a generalized version of the dummy property in Lemma~\ref{lem:null_classifier}. The main result of this section is Theorem \ref{thm:main} that characterizes our explanation method uniquely. 

\begin{restatable}{lemma}{lemmanullpropoerty}\label{lem:null_classifier}
If $I^v$ satisfies (S), (GE), (L) and (M) then if $m_S(T,v) = 0$ $\forall T \subseteq N\setminus S$ then $I^v(S) = 0$.
\end{restatable}
\begin{proof}
 We begin by showing a weaker claim: if $g$ is a {\em null game} where $g(S) = 0$ for all $S \subseteq N$, then $I^g(S) = 0$ for all $S \subseteq N$.
 For a given null game $g$, we have $m_S(T,g) = 0$ for all $S\subseteq N$ and $T\subseteq N \setminus S$. By symmetry, $I^g(S_1) = I^g(S_2)$ for all $S_1,S_2 \subseteq N$ with $|S_1|= |S_2|$, because for any permutation $\pi$, $g = \pi g$. 
 Also, by the Limit Condition(L), $I^g(N) = m_N(\emptyset,g) = 0$. Similarly, for all $i_1\neq  i_2 \in N$, $I^{g_{[i_1,i_2]}}(N\setminus i_1i_2 \cup [i_1,i_2]) = 0$. In fact, this property holds for all $V \subset N$: $I^{g_{[V]}}(N \setminus V \cup [V]) = 0$. 
 Now we use (GE) property for $I^{g_{[V]}}(N \setminus V \cup [V])$ by sequentially removing all $k\in V$ until it becomes a singleton. First, for all $k_1 \in V$;
 \begin{equation*}
     \begin{split}
       0 = I^{g_{[V]}}(N \setminus V \cup [V])  = I^{g_{[V\setminus k_1]}}(N \setminus V \cup [V\setminus k_1]) +\\ I^{g_{[V\setminus k_1]}}(N \setminus V \cup k_1) =  2I^{g_{[V\setminus k_1]}}(N \setminus V \cup [V\setminus k_1])   
     \end{split}
 \end{equation*}
The second equality holds because of symmetry (S) property for the game $g_{[V\setminus k_1]}$. 
Now for $k_2 \in V \setminus k_1$; we similarly use the (GE) property to obtain $I^{g_{[V\setminus k_1]}}(N \setminus V \cup [V\setminus k_1]) = 2I^{g_{[V\setminus k_1]}}(N \setminus V \cup [V\setminus k_1k_2]) $. We repeat this argument until only one element is left. We get $0 = I^{g_{[V]}}(N \setminus V \cup [V]) = 2^{|V| - 1} I^g(N\setminus V \cup k)$. This equality holds for all $V\subseteq N$ and all $k\in V$. Which shows that for all $S \subseteq N$, $I^g(S) = 0$.

Now, to prove the second part of the lemma, consider any game $v$. If $m_S(T,v) = 0$ for all $T\subseteq N\setminus S$, then $m_S(T,v) = m_S(T,g)$ for all $T\subseteq N \setminus S$, therefore by the Monotonicity property(M), $I^v(S) = I^g(S) = 0$, which concludes the proof.
\end{proof}

Before we prove the main theorem, we define important terms which will be helpful to prove the theorem. \newline 
Given a set of features $R \subseteq N$, we define the \emph{primitive game} $p^R$ as:
\begin{align*}
p^R(S)=
    \begin{cases}
      1, & \text{if } R\subseteq S \\
      0, & \text{else} 
    \end{cases}
\end{align*}
\newline
The set $\mathcal B^N$ of all Boolean functions forms a \emph{vector space}, and the set of primitive games $\mathcal P^N = \{p^R :  R \subseteq N \}$ of all primitive games is an \emph{orthonormal basis} of the vector space $\mathcal B^N$. 
Therefore, any Boolean function $v$ (In this context a cooperative game $v$) is uniquely represented by a linear combination of primitive games. 
In other words, given a cooperative game $v:2^N \to \R$, we can uniquely write as a linear combination of primitive games:
\begin{align}
    v(\cdot) = \sum_{R\subseteq N} c_R p^R(\cdot)\label{eqn:unique-decomp}
\end{align}    
The unique decomposition of cooperative games is useful for our characterization of high-dimensional model explanations. We also require some technical claims which will be useful for our characterization result (for proofs see supplementary material~\cite{supplementarymaterial}).

\begin{restatable}{theorem}{thmmain}\label{thm:main}
BII is the only high-dimensional model explanation satisfying (S),(GE),(L) and (M).
\end{restatable}

\begin{proof}
We first show that \BII satisfies all the properties. 
\BII trivially satisfies (S), (L) and (M). To show that it satisfies (GE), take any $S \subseteq N \setminus ij$
\begin{align*}
      I_\BII^v(S\cup i ) =& \frac{1}{2^{n-s-1}}\sum_{T\subseteq N\setminus(S\cup i)} m_{S\cup i} (T,v)\\
  = & \frac{1}{2^{n-s-1}}\sum_{T\subseteq N\setminus(S\cup ij)} [ m_S (T \cup i, v) - m_S (T, v )]\\
     & + \frac{1}{2^{n-s-1}}\sum_{T\subseteq N\setminus(S\cup ij)} [m_S(T \cup ij , v) - m_S (T\cup j, v)]    
\end{align*}

A similar calculation for $j$ shows that 
\begin{align*}
    I_\BII^v(S\cup j )  =&  \frac{1}{2^{n-s-1}}\sum_{T\subseteq N\setminus(S\cup ij)} [ m_S (T\cup j , v) - m_S (T, v )] \\
    &+ \frac{1}{2^{n-s-1}}\sum_{T\subseteq N\setminus(S\cup ij)} [m_S (T \cup ij,v)- m_S (T \cup i , v)]    
\end{align*}

Thus, $I_\BII^v(S\cup i ) +  I_\BII^v(S\cup j )$ equals
\begin{align}
    &\frac{1}{2^{n-s-2}} \times\sum_{T\subseteq N\setminus(S\cup ij)} [ m_S (T \cup ij, f) - m_S (T,f )]\label{eqn:gE}
\end{align}
Equation~\eqref{eqn:gE} shows that $I^{v_{[ij]}}(S\cup [ij]) = I^{v}(S\cup i) + I^{v}(S\cup j)$.

\BII satisfies the four axioms; to show that it {\em uniquely} satisfies them, we use the fact that $v$ can be uniquely expressed as the sum of primitive games;
\begin{align}
    v = \sum_{R\subseteq N} C_R p^R\label{eqn:basis}
\end{align}

We define the \textit{index} $\Gamma$ of a cooperative game $v$ to be the minimum number of terms in the expression of the form \eqref{eqn:basis}. We prove the theorem by induction on $\Gamma$. For $\Gamma=0$, in Lemma~\ref{lem:null_classifier}, $I^v(S)=0$ for all $S\subseteq N$, which coincides with the Banzhaf interaction index. 

If $\Gamma=1$ then $v = C_Rp^R$ for some $R \subseteq N$. Consider $S \not \subseteq R$; Proposition A1 in the supplementary material~\cite{supplementarymaterial} implies that $m_S(T,v)=0$ for all $T \subseteq N$, which in turn implies $I^v(S) = 0 = I_\BII^v(S)$. 
By Lemma A3 in the supplementary material~\cite{supplementarymaterial}, $I^v(R) =C_R$, which equals $I_\BII^v(R)$ by Proposition A2 in the supplementary material \cite{supplementarymaterial}. 
To complete the proof for the first inductive step, we need to show that for all $S \subseteq R$, $I^v(S) = I_\BII^v(S) = \frac{C_R}{2^{|R| - |S|}}$. 
If $S_1, S_2 \subseteq R$ and $s_1=s_2$ then by symmetry, $I^v(S_1) = I^v(S_2)$: we can define a permutation $\pi $ over $N$ such that $S_2$ bijectively maps to some $S_1$, and all $i \notin S_1\cup S_2:$ are invariant. By the Symmetry property $I^v(S_2) = I^{\pi v}(S_1)$; however, $\pi v = v$ because $v = C_Rp^R$.  

Now, consider any $i_1\neq i_2\in R$ and define $S:= R\setminus \{i_1,i_2\}$; by the GE property, we can write 
\begin{equation*}
    I^{v_{[i_1i_2]}}(S \cup [i_1i_2]) = I^{v}(S \cup i_1) + I^v(S\cup i_2)
\end{equation*}
which implies for any $Q\subset R$ with $|Q| = |R|-1$, $I^{v}(Q) = \frac{1}{2}I^{v_{[i_1i_2]}}(S \cup [i_1i_2])$. The reduced game $v_{[i_1i_2]}$ is
\begin{equation*}
    v_{[i_1i_2]}(S')=
    \begin{cases}
      C_R, & \text{if  } S\cup[i_1i_2] \subseteq S'\  \\
      0, & \text{else} 
    \end{cases}
\end{equation*}    
By Lemma A3 in the supplementary material~\cite{supplementarymaterial}, $I^{v_{[i_1i_2]}}(S \cup [i_1i_2]) = C_R$ and $I^{v}(Q) = \frac{C_R}{2}$. This property holds for all $T \subset N$, $v_{[T]}$; $I^{v_{[T]}}(N \setminus T \cup [T]) = C_R$. By inductively using the (GE) property, in a manner similar to Lemma A3 in the supplementary material~\cite{supplementarymaterial}, we show that $I^v(Q) = \frac{C_R}{2^{|R|-|Q|}}$. By Proposition A2 in the supplementary material~\cite{supplementarymaterial}, this coincides with Banzhaf interaction index concluding the first inductive step.

To complete the proof, assume that $I^v(S)$ coincides with the Banzhaf interaction index whenever the index of the game $v$ is at most $\Gamma = \gamma $. Suppose that $v$ has an index $\gamma+1$, and expressed as
\begin{equation*}
    v = \sum_{k=1}^{\gamma+1} C_{R_k}p^{R_k}
\end{equation*}
Let $R = \bigcap \limits^{\gamma +1}_{k=0} R_k$, and suppose that $S \not \subseteq R$. We define another game $w$:
\begin{equation*}
    w = \sum_{k:S \subseteq R_k} C_{R_k}p^{R_k} 
\end{equation*}
Since $S \not \subseteq R$, the index of $w$ is strictly smaller than $\gamma + 1$. We claim that for all $T\subseteq N\setminus S$; $m_S(T,v) = m_S(T,w)$. Indeed, consider any $T \subseteq N\setminus S$; $m_S(T,v)$ equals
\begin{align*}
    \sum_{L\subseteq S} (-1)^{|S|-|L|} v(T \cup L) =  \sum_{L\subseteq S} (-1)^{|S|-|L|} \sum_{k = 1}^{\gamma + 1} C_{R_k} p^{R_k}(T \cup L)&=\\
    \sum_{k = 1}^{\gamma + 1} \sum_{L\subseteq S} (-1)^{|S|-|L|}  C_{R_k} p^{R_k}(T \cup L)
            = \sum_{k = 1}^{\gamma + 1}C_{R_k} m_S(T,p^{R_k})&=\\
           \sum_{k:S \subseteq R_k} C_{R_k} m_S(T,p^{R_k}) =  m_S(T,w)
\end{align*}

The second-last equality holds by Proposition A1 in the supplementary material~\cite{supplementarymaterial}, hence by induction on $\Gamma$ and monotonicity(M) $I^v(S)$ coincides with \BII for all $S \not \subseteq R $.

It remains to show that $I^v(S)$ coincides with \BII when $S\subseteq  R$. for any $S\subseteq R$, consider any $i\in S$ and define $S':= S \setminus i$. Take any $j \in N$ such that $j_1 \in R_1\setminus R_2 \cup R_2\setminus R_1$. By the (GE) property, we can write 
\begin{equation}
    I^v(S) = I^{v_{[ij_1]}}(S'\cup [ij_1]) - I^v(S'\cup j_1)\label{ite_1}
\end{equation}

In Equation \eqref{ite_1}, $S'\cup j_1 \not \subseteq R$, therefore as previously shown, $I^v(S'\cup j_1)$ coincides with \BII for the game $v$. Consider the restricted game $v_{[ij_1]}$:
\begin{equation*}
    v_{[ij]} = \sum_{k=0}^{\gamma+1} C_{R_k} p_{[ij_1]}^{R_k\setminus ij_1\cup [ij_1]}
\end{equation*}
Consider $j_2\neq j_1 \in  R_1\setminus R_2 \cup R_2\setminus R_1$. By the (GE) property, 
\begin{equation}
    I^{v_{[ij_1j_2]}}(S'\cup [ij_1j_2]) = I^{v_{[ij_1]}}(S'\cup [ij_1])
    + I^{v_{[ij_1]}}(S'\cup j_2) \label{ite_2}
\end{equation}
In Equation \eqref{ite_2}, $S'\cup j_2 \not \subseteq \bigcap \limits^{I+1}_{k=1} R_k\setminus ij_1\cup [ij_1] $, therefore as we have shown before, $I^{v_{[ij_1]}}(S'\cup j_2) = I_\BII^{v_{[ij_1]}}(S'\cup j_2)$. Let us denote $T = R_1\setminus R_2 \cup R_2\setminus R_1$ and $T' = i\cup T $. By repeating this argument for all $j_3,\dots ,j_t \in T $ and exploiting the (GE) property for each $j_\ell$, we can write $I^v(S)$ as:
\begin{align}
    I^v(S) = I^{v_{[T']}}(S'\cup [T'])-I^v(S'\cup j_1) -  \sum_{l=1 }^t I^{v_{[ij_1\dots j_\ell]}}(S' \cup j_\ell) \label{eqn:ite_3} 
\end{align}

 All of the summands in \eqref{eqn:ite_3} coincide with \BII, because $S'\cup j_\ell \not \subseteq \bigcap \limits^{I+1}_{k=1} R_k\setminus ij_1\dots j_{\ell-1}\cup [ij_1\dots j_{\ell-1}]$  for all $\ell=1,\dots,t$. We can write the reduced game $v_{[T']}$ as 
\begin{align*}
 v_{[T']} = (C_{R_1}+C_{R_2})p_{[T']}^{(R_1 \cap R_2 \setminus i) \cup [T'] }+ \sum_{k=3}^{\gamma+1} C_{R_k} p_{[T']}^{(R_k\setminus T') \cup [T']}
\end{align*}

Thus, the index of the reduced game $v_{[T']}$ is strictly smaller than $\gamma+1$. By induction, $I^{v_{[T']}}(S' \cup [V])$ coincides with \BII for the reduced game $v_{[T']}$. $I^v(S)$ can be written as
\[
    I^v(S) = I^{v_{[T']}}(S'\cup [T']) - \sum_{l=1 }^t I^{v_{[ij_1\dots j_\ell]}}(S' \cup j_\ell)
\]
By using the (GE) property inductively, $I_\BII^v(S)$ can also be written in the same form, which implies that $I^{v}(S)$ coincides with the Banzhaf interaction index for all, $S \subseteq R$.
\end{proof}

\subsection{Explaining Our Model Explanations}\label{sec:explainingModelExpl}
Does the \BII measure make sense in the model explanation domain? This is purely a function of the strength of the axioms we set forth. 
Symmetry is natural enough: if a model explanation depends on the indices of its features then it fails a basic validity test. The index in which a feature appears has no bearing on the underlying trained model (ideally), nor does it affect the outcome.  

Recall that Generalized Efficiency requires that model explanations should be invariant under feature merging. In other words - artificially treating a pair of features as a single entity (while maintaining the same underlying model) should not have any effect on how feature behaviors are explained. Interestingly, Shapley values are not invariant under feature merging, a result shown by \citet{lehrer1988axiomatization}. The following examples illustrate what this entails in actual applications.

\begin{example}\label{ex:moviereviews}
Consider an sentiment analysis task where a model predict if a movie review was positive. In a preliminary step the text is parsed by a parser to be machine readable. This can be done in many different ways. For example the sentence ``This isn't a absolutely terrible movie'' Can be parsed as
\begin{center}
| This | isn't | a | absolutely | terrible | movie | . |
\end{center}
or as
\begin{center}
| This | is | n't | a | absolut | ely | terrible | movie | . |. 
\end{center}
Generalized Efficiency ensures that the influences of ``is'' and ``n't'' in the second version add up to the influence of ``isn't'' in the first. In other words, Generalized Efficiency ensures that the influence of features generated through different parsers behaves in a sensible manner. 
\end{example}

\begin{example}\label{ex:correlated-features}
Features might be ``merged'' in another situation when features that were readily available during the training of a model end up being costly to obtain during its deployment. 
If additionally these features are highly correlated with other features they might just be coupled. E.g. generally birds can fly, so the features \textsc{is\_bird} and \textsc{can\_fly} may simply be merged at prediction time\footnote{The authors are aware of the existence of ostriches, emus, penguins and the fearsome cassowary.}, to make it easier to enter information into a classifier. 
Again, Generalized Efficiency ensures that the influence of the merged feature relates in a natural way to the influence of the original features.
\end{example}

The \emph{Limit condition} normalizes the overall influence to be the discrete derivative of $v(\cdot,f)$ with respect to $N$. In other words, the total influence distributed to sets of features equals the total marginal effect of reverting features to their baseline values. 
This is an interesting departure from other efficiency measures. Shapley-based measures require efficiency with respect to $f(\vec x)$ (or variants thereof), i.e. the total amount of influence should equal the total value the classifier takes at the point of interest (or the difference between the classifier and the baseline value). We require that the total influence equals the (discretized) rate in which features change the outcome. This makes \BII more similar in spirit to gradient based model explanations, which are often used for generating model explanations in several application domains \cite{simonyan2013deep}. 

\emph{Monotonicity} is a very natural property in the model explanation domain: if a set of features has a greater effect on the value $f(\vec x)$, this should be reflected in the amount of influence one attributes to it. This has already been established in prior works, for Shapley-based measures \cite{datta2016algorithmic}. 
However, this property does not naturally generalize when using Shapley-based high-dimensional model explanations. \citet{agarwal2019new} propose a novel generalization of the Shapley value to high-dimensional model explanations, which fails monotonicity for smaller interactions (size of $<k$) for $k$-th order explanations, however, interactions of size $k$ follow monotonicity. Example~\ref{ex:fx1x2x3} highlights issues that may arise when monotonicity is not preserved.
\begin{example}\label{ex:fx1x2x3}
Given a function $f_c(x_1,x_2,x_3) =  c x_1 x_2 x_3$ with $c >0$ defined on binary input space (for example, $f$ is the result of an image classification task where $x_i$ denotes the presence/absence of particular super-pixel). 
We assume that the baseline is $\vec x' = (0,0,0)$. Thus, $v(S,\vec x,\vec x',f_c) = 0$ if $\{2,3\} \not\subseteq S$, resulting in $v(\{2,3\},f_c) = 0$ and $v(N,f_c) = v(\{1,2,3\},f_c)= cx_1x_2x_3 $. 
What is the interaction value between $1$ and $2$? Intuitively $\{1,2\}$ offer some degree of interaction that monotonically grows as $c$ increases. Moreover, it is easy to see that $v(\cdot,f_c) \ge v(\cdot,f_{c'})$ whenever $c \ge c'$. Set-QII fails to satisfies the monotoncity, and fails to capture the interaction between $\{1,2\}$ for any $c$. Set-QII($\{1,2\},S,f_c)=0$ for all $c$.  

Similarly, the Shapley-Taylor interaction index for $S = \{1,2\}$ and $k = 3$ is $0$ as it does not follow the monotonicity property, however for $k = 2$ it satisfies monotonicity and interaction value for $\{1,2\}$ is $ \frac c 3$. The \BII value for $\{1,2\}$ is $c$.

\end{example}
 
In Section~\ref{sec:geo}, we show that \BII can be interpreted as a polynomial approximation, offering additional intuition as to why our explanation method is \emph{optimal}.

\section{Geometrical Interaction and \BII }\label{sec:geo}
The geometry of model explanations is relatively well understood for attribute-based methods \cite{ribeiro2016should,sliwinski2019axiomatic,lundberg2017unified}; 
\citet{ribeiro2016should} view linear explanation methods as local linear approximations of $f(\cdot)$ around a point of interest $\vec x$. 
Linear model explanations can be thought of as functions taking the following form: $g(\Vec{x}) = \phi_0 + \sum_{i = 1}^N \phi_i x_i$, where $\phi_i$ captures the importance of feature $i$. 
Taking this interpretation, linear model explanations can be `objectively' bad explanations: they are a poor local approximation of the underlying model $f$ as black-box model can be highly non-linear around $\poi$ \cite{sliwinski2019axiomatic, tsang2017detecting}. 
In order to better capture the behavior of a black-box model $f$, we approximate it using higher-order polynomials. 
For better visualization, we first assume that the black-box model $f:\{0,1\}^N \rightarrow \mathbb{R}$ takes a binary input vector mainly referred to as the humanly understandable feature representation  \cite{lundberg2017unified, ribeiro2016should}. 
 First, we start with quadratic approximation of the black-box model $f(\cdot)$, $$g(\vec{x}) = \phi_0 + \sum_{i = 1}^N \phi_i x_i + \sum_{i<j} \phi_{i,j}x_ix_j $$
 
In the above equation, $\phi_{i,j}$ captures the interaction between features $i$ and $j$; 
$\phi_i,\phi_j$ capture the importance of $i$ and $j$. Thus, it is not unreasonable to assume that $\phi_{i,j}$ captures the pure interaction effect of $i$ and $j$: we can delegate the singular effects to $\phi_i$, having the resultant coefficient of $x_ix_j$ capture the `pure' interaction between $i$ and $j$. This also connects with the idea of the {\em Statistical Interactions} \cite{fisher1992statistical,gelman2005analysis}.
For instance, consider a sentiment analysis problem, both the tokens ``bad'' and ``not'' have negative influence on the prediction. However, when they are present together as ``not bad'', their influence is positive. In this simple example, it would be desirable to have $\phi_{\text{\{``not'',``bad''\}}}>0$. The idea of higher order interactions can be extended similarly.
 
Consider a global polynomial approximation of $f(\cdot)$ by a $k$-degree polynomial: 
\begin{equation}
     g^k(\Vec{x}) = \phi_0 + \sum_{S'\subset N: |S'|<k} \left( \phi_{S'}\prod_{j\in S' }x_j \right) + \sum_{S\subseteq N; |S|=k} \left( \phi_S\prod_{j\in S }x_j \right). \label{eqn:k-thapprox}
\end{equation}
Again, to capture interaction among the set of features $S$ such that $|S| = k$, we should remove all internal interaction effects captured by $\phi_{S'}$ for $S' \subset S$ as we intend to capture the  
The polynomial $g^k$ is meant to locally approximate $f$ around the POI $\vec x$; what is the best approximation? Finding the best fitting polynomial of the highest possible degree seems like a natural objective. However, taking this approach runs the risk of ignoring lower order feature interactions and their possible effects; this is shown in Example \ref{ex:lower-order-interaction}.

\begin{example}\label{ex:lower-order-interaction}
Consider the degree 3 polynomial studied in Example \ref{ex:fx1x2x3}, $f(x_1,x_2,x_3) = cx_1x_2x_3$ with the baseline set to $(1,1,1)$ (rather than $(0,0,0)$ as was the case in Example \ref{ex:fx1x2x3}). 
The best approximation to $f$ is clearly itself. However, if we do so, then the interaction coefficients for variable pairs will be zero. This is arguably undesirable: for example, if $\vec x = (0,0,1)$, then $x_3$ has virtually no impact (it is already set at the baseline). Similarly, $x_1$ and $x_2$ have little individual effect, but do have significant joint effect - it is only when both are set to $1$ that we observe any change in label.  
 \end{example}
 
Now we formally define the optimization problem to find the ``best'' $k$-degree polynomial approximation of the black-box model $f(\cdot)$ globally. 
Let $\mathcal{P}^k$ be the set of $k$-degree polynomials.
We are interested in finding a polynomial $g^k_f(\cdot)$ which globally minimizes the quadratic loss between $f(\vec x)$ and $g(\vec x)\in \mathcal{P}^k$ for all $\vec x \in 2^N$, i.e.
\begin{equation}
    g^k_f(\vec x) = \argmin_{g(\cdot) \in \mathcal{P}^k} \sum_{\vec x \in 2^N} [f(\vec x) - g(\vec x)]^2 .\label{eqn:opt_poly}
\end{equation}
The interaction among features in $S$ with $|S| = k$ is measured as the coefficient of $\prod_{i \in S} x_i$ in the least square approximation of $f(\cdot)$ with a polynomial of degree $k$. Theorem~\ref{thm:opt_poly_approximation} shows that this geometrical definition of feature interaction coincides with the Banzhaf interaction index. 
\begin{theorem}\label{thm:opt_poly_approximation}
Let $g^k_f(\vec x)$ be the $k$-degree solution of the optimization problem in Equation~\eqref{eqn:opt_poly}. Then the coefficients of $\prod_{i \in S}x_i$ for $|S|=k$ in $g^k_f(\vec x)$ is given by $I_{\texttt{BII}}^v(S)$. 
\end{theorem}
The proof of the Theorem is a direct corollary of \cite[Theorem~4.2]{hammer1992approximations}. Theorem~\ref{thm:opt_poly_approximation} shows that k-th order interaction obtains via least square approximation of the black-box model using k-th order polynomial coincides with \BII and extend a geometrical argument for ``good" feature-based explanation to feature interactions. Moreover, the least square approximation of the model via polynomial also resonates with the Statistical interaction among features as we already discussed. Therefore \BII is not only obtained from strong axiomatic but also derived from an intuitive optimization problem. 
\section{Discussion and Future Work}
%
Designing \emph{provably sound} higher-order explanations for machine learning models in high stake domains is an important problem. Ideally, we want to design explanations we can trust.
We offer a variety of reasons to trust in \BII: it uniquely satisfies a set of natural properties --- if we believe that these are sensible then our job is done. That said, the authors believe that a normative approach is critical in the design of fair, transparent AI solutions. Rather than debating approaches, one should debate the fundamental properties they are guaranteed to satisfy. 
Deriving model explanations from norms offers an intuitive justification for the chosen interaction measure, which fosters trust in the explanation method. 
As shown in Section \ref{sec:experiments}, some of our fundamental properties are absolutely critical for the design of sensible explanations. 
Even if one adheres to an optimization-based approach to explanation design, we show that \BII is the optimal solution to a natural objective. 

While \BII certainly satisfies several design criteria, it is not without issues. The first challenge is computational. Feature-based explanations are computationally intensive as it is, so it should be no surprise that computing higher-order interactions comes at a higher cost. There are, however, good reasons to believe that \BII can be computed efficiently on simpler ML model classes, and that low error approximations can be found quickly. This, we believe, is an important direction for future work.   

\section*{Acknowledgement}
This research was supported by an NRF Research Fellowship R-252-000-750-733 and the [AISG: R-252-000-A20-490] the National Research Foundation, Singapore under its AI Singapore Programme (AISG Award No: AISG-RP-2018-009). This work was done when the first and the third authors were affiliated with the National University of Singapore. The authors thank GAIW 2020, WHI 2020, NeurIPS 2020, and FAccT 2021 reviewers for their useful comments. 

\nocite{chalkiadakis2011computational,bachrach2008approximating}
\bibliographystyle{ACM-Reference-Format}
\bibliography{abb,interaction-index-biblio}

\appendix
\section{Useful Technical Results}\label{app:techresult}

\begin{restatable}{proposition}{propforprimitive}\label{prop:primitive}
Given a primitive game $p^R$, for any $S \subseteq N$:
if $S \not \subseteq R$ then $\forall T\subseteq N\setminus S $, $m_S(T,p^R) = 0$. In particular, $I_\BII^{p^R}(S) = 0$.
\end{restatable}
\begin{proof}
Suppose that $S \not \subseteq R$. We distinguish between two cases. 
\begin{description}
\item[Case 1:] $R\not \subseteq T\cup S$. In this case, $\forall L \subseteq S$, $T \cup L$ does not contain $R$, thus $p^R(T\cup L) = 0$ which implies $m_S(T,p^R) = 0$.

\item[Case 2:] $R \subseteq T\cup S$. In this case, $m_S(T,p^R)$ equals

\begin{align*}
\sum_{L\subseteq S} (-1)^{|S|-|L|}p^R(L \cup T)=&\sum_{L\subseteq S:S \cap R \subseteq L} (-1)^{|S|-|L|} =\\
\sum_{L\subseteq S\setminus R} (-1)^{|S|-|S\cap R|-|L|}=& \sum_{L\subseteq S\setminus R} (-1)^{|S\setminus R|-|L|}=\\
\sum_{k=0}^{|S\setminus R|} (-1)^k \binom{|S\setminus R|}{k} = &0
\end{align*}
\end{description}
Thus in either case $m_S(T,p^R) = 0$, and we are done.
\end{proof}

We note that Proposition \ref{prop:primitive} immediately holds for any game that is a scalar multiple of a primitive game, i.e. for any $v = c\times p^R$, and any $S \not \subseteq R$, $I_\BII^v(S) = 0$.

\begin{restatable}{proposition}{propforBIIprimitive}\label{prop:primBanzhaf}
Given a primitive game $p^R$, for any $S \subseteq N$:
If $v = c\times p^R$, and $S \subseteq R$ then 
$$I_{\BII}^v(S) = \frac{c}{2^{|R|-|S|}}.$$
\end{restatable}
\begin{proof}
Since $S \subseteq R$, then for any $L \subset$ and any $T \subseteq N\setminus S$, $v(L\cup T) = 0$. Therefore,  
\begin{align}
    I_\BII^v(S) &=\frac{1}{2^{n-|S|}} \sum_{T\subseteq N \setminus S}\sum_{L \subseteq S} (-1)^{|S|-|L|} v(L \cup T)\notag\\
    &= \frac{1}{2^{n-|S|}} \sum_{T\subseteq N \setminus S}v(S \cup T)\label{eqn:primBanzhaf-summand}
\end{align}
Next, if $T\subseteq N\setminus S$ does not contain $R\setminus S$, then $v(S\cup T) = 0$. Therefore, the summand in \eqref{eqn:primBanzhaf-summand} equals
\begin{align*}
    \sum_{T\subseteq N \setminus R}v(R \cup T) = c\times \sum_{T\subseteq N \setminus R}1 = c\times 2^{n - |R|} 
\end{align*}
Plugging this back into \eqref{eqn:primBanzhaf-summand}, we obtain the desired result.
\end{proof}

Next, let us characterize how influence measures satisfying our axioms behave on primitive games. Note that Lemma \ref{lem:for_primitive_classifier} offers a special case of Proposition \ref{prop:primitive} for \emph{any} influence measure, rather than just for \BII.

\begin{restatable}{lemma}{lemmaprimitiveaxioms}\label{lem:for_primitive_classifier}
If $I^v$ satisfies (S), (GE), (L) and (M) then for $v=c\times p^R$, $I^v(R) = c$
\end{restatable}

\begin{proof}
We prove this lemma by inductively removing a feature $k\in N\setminus R$ and using the (GE) property at each step. 
Take any feature $i\in R$ and remove it from $R$ and define $S:=R\setminus \{i\} $. Now for any feature $j_1 \in N\setminus R \neq i$, by (GE) property we can write,
\begin{equation*}
    I^{v_{[ij_1]}}(S \cup [ij_1]) = I^{v}(S \cup i) + I^v(S\cup j_1)
\end{equation*}
Since $S \cup j_1\not \subseteq R$, by Proposition~\ref{prop:primitive}, $m_{S\cup j_1}(T,v) = 0$ for all $T \subseteq N \setminus \{S\cup j_1\}$. 
Therefore by Lemma~\ref{lem:null_classifier}, $I^v(S\cup j_1) = 0$, which yields $I^v(R) =I^{v_{[ij_1]}}(S \cup [ij_1]) $. We next remove $j_2\neq j_1 \in N\setminus R $, and again invoke the (GE) property:
\begin{equation*}
  I^{v_{[ij_1j_2]}}(S \cup [ij_1j_2]) = I^{v_{[ij_1]}}(S \cup [ij_1]) + I^{v_{[ij_1]}}(S\cup j_2)  
\end{equation*}
It is easy to check that $I^{v_{[ij_1]}}(S\cup j_2) = 0$ by Proposition~\ref{prop:primitive} and Lemma~\ref{lem:null_classifier} for the reduced game $v_{[ij_1]}$. 
\[    v_{[ij_1]}(S')=
    \begin{cases}
      c, & \text{if  } S\cup[ij_1] \subseteq S'\\
      0, & \text{else} 
    \end{cases}
\]
$m_{S \cup j_2}(T,v_{[ij_1]}) = 0$ for any $T\subseteq \{N\setminus \{i,j_1\}\cup [ij_1]\} \setminus \{S \cup j_2\} $. This implies that $I^v(R) =I^{v_{[ij_1j_2]}}(S \cup [ij_1j_2])$. By repeating this argument for all $j\in N\setminus R$, we will have $I^v(R) =I^{v_{[N\setminus S]}}(S \cup [N\setminus S])$. We can write the reduced game $v_{[N\setminus S]}$ as   
\[   
    v_{[N\setminus S]}(S')=
    \begin{cases}
      c, & \text{if } S' = S \cup [N\setminus S]\  \\
      0, & \text{else.} 
    \end{cases}
\]
By the Limit (L) property, 
$$I^{v_{[N\setminus S]}}(S \cup [N\setminus S]) = m_{S \cup [N\setminus S]}(\emptyset,v_{[N\setminus S]}) = c,$$ 
which concludes the proof
\end{proof}

\section{Additional Example}\label{app:examples}

 In the following example, we demonstrate that the Shapley interaction index can be misleading in simple situations. We consider the general majority classification function which exhibits pairwise feature interaction. Shapley interaction indices fail to capture these interactions. 
 Moreover, Shapley-Taylor interaction indices fail to capture the sign of pairwise interactions for the same function. 
 \begin{example} \cite{agarwal2019new}
 Consider a classification function whose input space is binary. Let the classification function $f$ be: $f(x_1,\dots,x_n) = 1$ iff $\sum_ix_i \ge k$ and $0$ elsewhere with the baseline vector $\vec x' = \vec 0$. Thus, $v(S,\vec 1, \vec x') = 1 $ iff $|S|\geq k$ and $0$ otherwise. 
 For $k=\frac n 2$, the function coincides with the majority function discussed in the cooperative game theory literature. 
 Clearly, there exists pairwise interaction among features, however, the Shapley interaction value for each pairwise feature is $0$. 
 In contrast, the pairwise Banzhaf interaction index for any feature pair $\{i,j\}$ is $c_k (2k - (n+3))$ where $c_k > 0$. 
 Pairwise interaction is negative when $2k < n+3$. This can be explained by the following argument: the number of winning coalitions containing  $\{i,j\}$ is $\binom{n-2}{k-2}$, and the number of winning coalitions that do not contain $\{i,j\}$ $\binom{n-2}{k}$, which is higher for smaller $k$. 
 This shows that $\{i,j\}$ has more interaction effect for the output $0$. 
 On the other hand, the Shapley-Taylor pairwise interaction index is $\frac{2}{n(n-1)}$, and fails to capture the sign of the interaction index. 
\end{example}
\section{Computational Complexity}\label{sec:computations}
Computing \BII exactly --- even for individual features, where it coincides with the Banzhaf index --- is intractable (see \cite[Chapter 4]{chalkiadakis2011computational} for a general overview); however, due to its probabilistic nature, \BII can be well-approximated via random sampling for individual features \cite{bachrach2008approximating}. The sampling approach is widely used for approximating cooperative game theoretic influence measures for the individual features \cite{datta2016algorithmic,lundberg2017unified}. \newline
To approximate the pairwise interaction for \BII, we can decompose pairwise \BII as following:
\begin{align}
    I^v_{\BII}(ij) &= \frac{1}{2^{n-2}} \sum_{T\subseteq N \setminus ij} (v(T\cup ij) - v(T\cup i) - v(T\cup j) + v(T))\notag\\
    &= I^{v_{[ij]}}_{\BII}([ij])  -  I^{v(N\setminus i)}_{\BII}(j) -  I^{v(N\setminus j)}_{\BII}(i) \label{eqn:BIIdecompsistion}
\end{align}
This shows that there exists an efficient sampling scheme which approximates the pairwise \BII: each term in Equation~\ref{eqn:BIIdecompsistion} can be approximated by polynomially many samples similar to the \cite{datta2016algorithmic,lundberg2017unified}. Moreover, we further note that the approximation algorithms discussed for approximating Shapley interaction index in \citet{lundberg2018consistent} and polynomial algorithms for the tree-based decision model can be implemented for \BII as well. 
\section{Broader Impact}
The recent wave of research on explainable machine learning is motivated by the disparity between the success of machine learning, and its pronounced lack of public accountability. This work takes the next natural step towards a more general theory of model explanations. Our work treats explanation complexity as a \emph{tunable parameter}: we can generate explanations of varying degree of complexity, as is appropriate for different contexts and stakeholders. The authors view this as a \emph{positive} effect on society.

Machine Learning models have been shown to potentially leak information about the training data \cite{nasr2019comprehensive} and there is some recent work that shows how explanations can be used to reconstruct models \cite{milli2018model} or recover user data \cite{shokri2019privacy}. 
There is a growing body of evidence that explanations make models more prone to leak private training data; such leakage could undermine public trust in machine learning models, the very same trust that model explanations try to establish. It is only natural to assume that our high-dimensional model explanations are more vulnerable to adversarial exploitation, as they release more information per query. This is a potential societal \emph{risk}. 
Advances in developing differentially private machine learning models \cite{abadi2016deep} offer some hope that privacy preserving model explanations can be developed. In fact, QII is shown to be differentially private (with respect to the sample set) \cite{datta2016algorithmic}.

Finally, our work offers several {\em mathematical} guarantees of explanation quality, but does not focus on its \emph{actual adoptability}. There is currently little evidence that mathematically justified model explanations, as elegant as they may be, are acceptable by human stakeholders, or actually help humans understand model decisions. 
In fact, recent work suggests that model explanations can result in overconfidence of software developers in their understanding of models \cite{kaur2020interpreting}. We believe that an axiomatic approach, like the one chosen in the paper helps ensure that stakeholders understand the extent to which an explanation can be useful. This makes explanation methods such as ours more \emph{trustworthy}. However, we also believe that before model explanation methods are deployed, the research community needs to clearly specify which tools can be used for which tasks, and that end users understand their limitations. Some parts of the community are already moving towards a more user-centric direction and believe that this can help overcome existing shortcomings~\cite{kim2017Iiterpretability, lipton2016interpretability,wang2019explain}. 
\section{Experimental Evaluation}\label{sec:experiments}
In this section, we experimentally illustrate the need for high dimensional model explanations and the axioms discussed above. 
We consider two-well known benchmark datasets; the \emph{UCI Adult dataset} \cite{adultdata} tries to predict whether an individual's income exceeds \$50,000/year based on census data, and the \emph{IMDB movie review dataset} \cite{maas-EtAl:2011:ACL-HLT2011} tries to predict the positive/negative sentiment of movie reviews based on their word content.

For the Adult dataset we train a 3-level decision three (see Figure~\ref{fig:tree}) which achieves an 84\% test accuracy. We use this easily comprehensible model to show the shortcomings of existing methods and the importance of our axioms. We also train a random forest with 50 estimators and a maximum depth of 10 that achieves a slightly higher accuracy of 86\%\footnote{We use keras library to train decision tree and random forest \url{https://github.com/keras-team/keras}.}.

 \begin{figure*}
    \centering
    \if\compileFigures1
    \tikzset{
        ellC/.style={
            color of colormap={#1},
           fill=.,
        },
}

\begin{tikzpicture}[
    node/.style={%
    draw, rounded rectangle,
    minimum height=5mm,
    },
    leaf/.style={%
    draw, rounded rectangle,
    align=center
    }
  ]
  
  \newdimen\nodeDist
\nodeDist=35mm

    \def\nodeAngle{70};
    \def\nodeAngleTwo{58};
    \def\nodeDistTwo{20mm};
    \def\nodeAngleThree{40};
    \def\nodeDistThree{12mm};
    \node [node] (0) {\small $\MS$};
    \path (0) ++(-90-\nodeAngle:\nodeDist) node [node] (1) {\small $\CG$ $>$ 7073};
    \path (0) ++(-90+\nodeAngle:\nodeDist) node [node] (8) {\small $\EL$ $>$ 12};
    
    \path (1) ++(-90-\nodeAngleTwo:\nodeDistTwo) node [node] (2) {\small \small $\EL$ $>$ 12};
    \path (1) ++(-90+\nodeAngleTwo:\nodeDistTwo) node [node] (5) {\small $\Age$ $>$ 20};
    \path (8) ++(-90-\nodeAngleTwo:\nodeDistTwo) node [node] (9) {\small $\CG$ $>$ 5095};
    \path (8) ++(-90+\nodeAngleTwo:\nodeDistTwo) node [node] (12) {\small $\CG$ $>$ 5095};
    
    \path (2) ++(-90-\nodeAngleThree:\nodeDistThree) node [leaf,ellC=30,text=white] (3) {$0.03$}; 
    \path (2) ++(-90+\nodeAngleThree:\nodeDistThree) node [leaf,ellC=150,text=white] (4) {$0.15$}; 
    \path (5) ++(-90-\nodeAngleThree:\nodeDistThree) node [leaf,ellC=0,text=white] (6) {$0.00$}; 
    \path (5) ++(-90+\nodeAngleThree:\nodeDistThree) node [leaf,ellC=960,text=black] (7) {$0.98$}; 
    \path (9) ++(-90-\nodeAngleThree:\nodeDistThree) node [leaf,ellC=300,text=white] (10) {$0.30$}; 
    \path (9) ++(-90+\nodeAngleThree:\nodeDistThree) node [leaf,ellC=980,text=black] (11) {$0.98$}; 
    \path (12) ++(-90-\nodeAngleThree:\nodeDistThree) node [leaf,ellC=670,text=black] (13) {$0.67$}; 
    \path (12) ++(-90+\nodeAngleThree:\nodeDistThree) node [leaf,ellC=1000,text=black] (14) {$1.00$}; 

    \draw (0) -- (1) node [above left,pos=0.5] {No}(0);
    \draw (0) -- (8) node [above right,pos=0.5] {Yes}(0);
    
    \draw (1) -- (2) node [above left,pos=0.85] {No}(0);
    \draw (1) -- (5) node [above right,pos=0.85] {Yes}(0);
    \draw (8) -- (9) node [above left,pos=0.85] {No}(0);
    \draw (8) -- (12) node [above right,pos=0.85] {Yes}(0);
    
    \draw (2) -- (3) node [left,pos=0.5] {No}(0);
    \draw (2) -- (4) node [right,pos=0.5] {Yes}(0);
    \draw (5) -- (6) node [left,pos=0.5] {No}(0);
    \draw (5) -- (7) node [right,pos=0.5] {Yes}(0);
    \draw (9) -- (10) node [left,pos=0.5] {No}(0);
    \draw (9) -- (11) node [right,pos=0.5] {Yes}(0);
    \draw (12) -- (13) node [left,pos=0.5] {No}(0);
    \draw (12) -- (14) node [right,pos=0.5] {Yes}(0);
    
\end{tikzpicture}
    \else
    \includegraphics[]{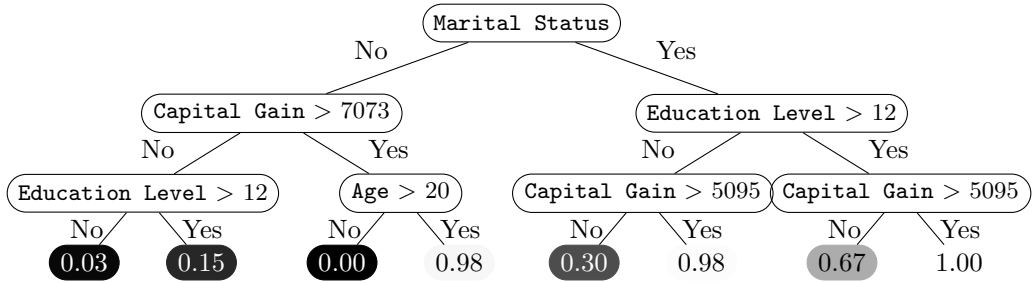}
    \stepcounter{figuerNumber}
    \fi
    \caption{A three-level decision tree trained on the adult dataset, achieving 84\% test accuracy.}
    \label{fig:tree}
\end{figure*}

For the IMDB dataset we train the BERT language model\footnote{We use pertained BERT model for the IMDB dataset available at \url{https://github.com/artemisart/bert-sentiment-IMDB}.} \cite{devlin2018bert}, which achieves $89\%$ evaluation accuracy; the BERT-Base model we use has 110M parameters, making it virtually incomprehensible to humans.

Like many other model explanation methods, \BII assumes the existence of a baseline, i.e. a vector having ``typical'' values which can be used as a reference for the point of interest when creating the explanation. In the Adult dataset, we assume the baseline for continuous features to be their median value; categorical features are one-hot encoded, and the baseline is assumed to be zero. For the IMDB dataset, each word is a distinct feature, and the baseline is an empty sentence (i.e. all features set to $0$).
\subsection{Baseline Explanation Frameworks}
We compare \BII with an interaction measure based on the Shapley interaction index (SII) \cite{grabisch1999axiomatic} and a metric that falls into the framework proposed in \cite{datta2016algorithmic} which we call Set-QII. 
For all methods we consider a version using the baseline i.e. the \BII and SII values are the respective game-theoretic measures when ``playing'' the game of setting features to the baseline. Set-QII measures the influence of a (set of) features by measuring the difference of the function value on the point of interest and a single another point where the values of this (set of) feature(s) is set to the baseline. For all three metrics, we consider model explanations capturing single feature influences and pairwise feature interactions. We briefly discuss the computational complexity in the supplementary material~\cite{supplementarymaterial}. 

To further illustrate the importance of measuring interactions we also study the well known \LIME method \cite{ribeiro2016should} which only generates individual feature importance. 

We also consider a naive approach to create feature interactions which we call the \emph{additive expansion}. The additive expansion of a feature-based explanation methods is computing joint influence of feature $i$ and $j$ (interaction among features $i$ and $j$) by simply adding individual influence of features $i$ and $j$, $I(ij) = I(i) + I(j)$. This is a naive way to compute the joint effect of features when only individual feature importance is provided. Via additive expansion we can create a feature-interaction measure from \LIME, as well as ad-hoc feature-interaction measures from the feature-importance versions of Banzhaf, Shapley and QII.   

 \subsection{The Importance of Interactions}

While seemingly obvious, we begin our empirical investigation with a simple question: are there instances where high-dimensional model explanations are needed? In other words, is it possible that feature-wise explanations offer sufficient explanations in practice? Our first experiment indicates that high-dimensional explanations matter. In Figure~\ref{fig:intro_bii}, we visualize interactions for the movie review described in the Introduction (Section~\ref{sec:intro}). 
We compute pairwise interactions using \BII among features for the BERT model. We represent the explanations generated by \BII as a symmetric matrix, where the $i$-th diagonal element represents the influence of feature $i$ in the decision tree and $(i,j)$-th element of the matrix for $i<j$ represents the interaction between $i$ and $j$.     

Comparing Figures~\ref{fig:intro_bii} and~\ref{fig:intro_lime}, it is immediately clear why feature importance offers additional insights. The values generated by \LIME are indicative of a negative review (whereas the review we examine is positive), and look somewhat odd: `not' has a strong positive influence. On the other hand, the interactions make it clear that BERT has picked up on the fact that `not bad' indicates a good review. 
We compare additional \BII feature interactions for the BERT model and the IMDB classification task in the supplementary material~\cite{supplementarymaterial} (Figures 1 to 5). \LIME and \BII mostly agree for the individual feature importance; however, \BII captures non-trivial interactions among pairs of words that are impossible to observe only from feature-based explanation.

\pgfplotsset{colormap={RdBu}{ rgb255=(103,0,31) rgb255=(176,23,42) rgb255=(214,96,77) rgb255=(243,163,128) rgb255=(253,219,199) rgb255=(246,246,246) rgb255=(209,229,240) rgb255=(144,196,221) rgb255=(67,147,195) rgb255=(32,100,170) rgb255=(5,48,97) }}

\begin{figure}
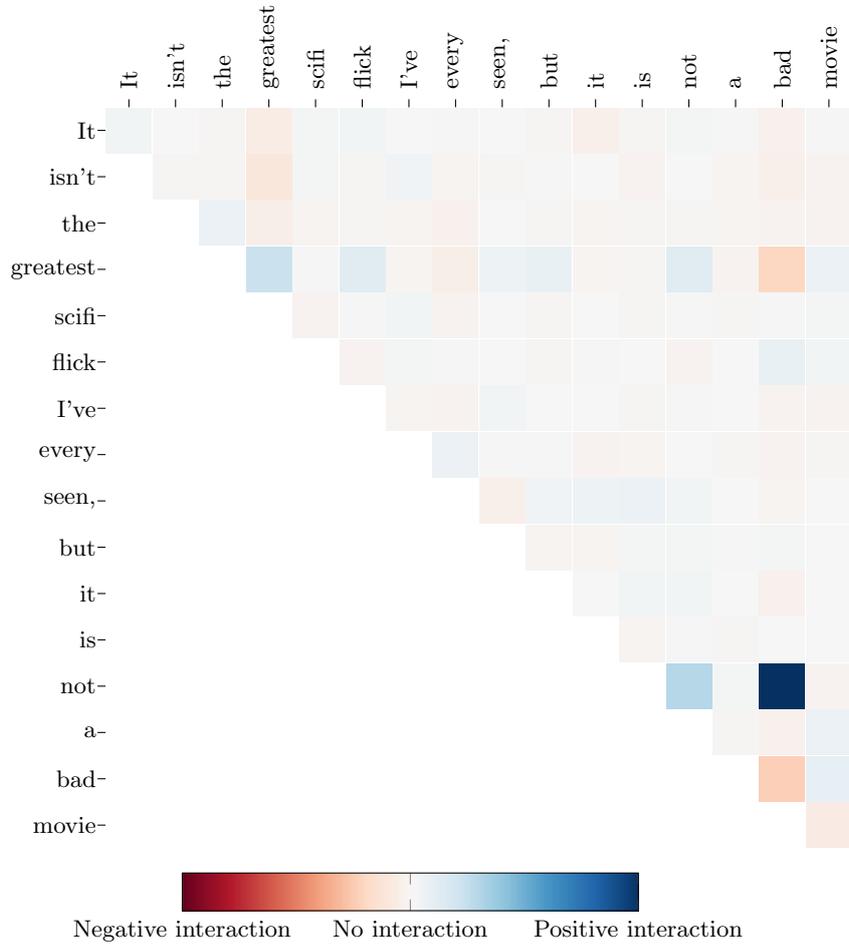

    \centering
    \if\compileFigures1
    \input{figures/heatmapDefinition.tex}
    \def\labels{{"It", "isn't", "the", "greatest", "scifi", "flick", "I've", "every", "seen,", "but", "it", "is", "not", "a", "bad", "movie", }}
\tikzset{cellColor/.style={color of colormap={#1},fill=.}}
\def\width{.6}
\def\twidth{1.5}
\def\height{.6}

\begin{tikzpicture}
\def\values{{
{0.03, 0.00, -0.01, -0.07, 0.02, 0.03, -0.00, 0.01, -0.00, -0.01, -0.05, -0.01, 0.02, 0.01, -0.04, 0.01, },
{0.00, -0.01, -0.01, -0.11, 0.02, -0.01, 0.04, -0.02, -0.01, 0.01, -0.00, -0.03, 0.00, -0.02, -0.05, -0.03, },
{0.00, 0.00, 0.06, -0.06, -0.02, -0.01, -0.02, -0.04, 0.00, -0.01, -0.02, -0.01, -0.01, -0.02, -0.03, -0.03, },
{0.00, 0.00, 0.00, 0.22, 0.01, 0.12, -0.02, -0.06, 0.05, 0.07, -0.02, -0.01, 0.12, -0.03, -0.21, 0.06, },
{0.00, 0.00, 0.00, 0.00, -0.03, 0.01, 0.03, -0.03, 0.00, -0.01, 0.00, -0.01, 0.01, -0.01, 0.01, 0.02, },
{0.00, 0.00, 0.00, 0.00, 0.00, -0.03, 0.02, 0.01, 0.00, -0.01, 0.01, 0.00, -0.03, 0.00, 0.07, 0.03, },
{0.00, 0.00, 0.00, 0.00, 0.00, 0.00, -0.02, -0.03, 0.03, 0.00, -0.00, -0.01, 0.01, 0.00, -0.03, -0.03, },
{0.00, 0.00, 0.00, 0.00, 0.00, 0.00, 0.00, 0.06, 0.01, 0.01, -0.03, -0.02, 0.00, -0.01, -0.03, -0.01, },
{0.00, 0.00, 0.00, 0.00, 0.00, 0.00, 0.00, 0.00, -0.05, 0.04, 0.05, 0.06, 0.03, 0.00, -0.02, 0.00, },
{0.00, 0.00, 0.00, 0.00, 0.00, 0.00, 0.00, 0.00, 0.00, -0.02, -0.02, 0.02, 0.02, 0.01, 0.02, 0.00, },
{0.00, 0.00, 0.00, 0.00, 0.00, 0.00, 0.00, 0.00, 0.00, 0.00, 0.00, 0.03, 0.03, 0.00, -0.04, -0.00, },
{0.00, 0.00, 0.00, 0.00, 0.00, 0.00, 0.00, 0.00, 0.00, 0.00, 0.00, -0.02, 0.01, -0.01, -0.00, 0.00, },
{0.00, 0.00, 0.00, 0.00, 0.00, 0.00, 0.00, 0.00, 0.00, 0.00, 0.00, 0.00, 0.29, 0.02, 1.00, -0.03, },
{0.00, 0.00, 0.00, 0.00, 0.00, 0.00, 0.00, 0.00, 0.00, 0.00, 0.00, 0.00, 0.00, -0.01, -0.04, 0.06, },
{0.00, 0.00, 0.00, 0.00, 0.00, 0.00, 0.00, 0.00, 0.00, 0.00, 0.00, 0.00, 0.00, 0.00, -0.24, 0.09, },
{0.00, 0.00, 0.00, 0.00, 0.00, 0.00, 0.00, 0.00, 0.00, 0.00, 0.00, 0.00, 0.00, 0.00, 0.00, -0.08, },
}}
\Heatmap{15}{\small}
\end{tikzpicture}\\
\begin{tikzpicture}
\begin{axis}[
    hide axis,
    scale only axis,
    height=0pt,
    width=0pt,
    colorbar horizontal,
    point meta min =-1,
    point meta max =1,
    colorbar style={
        width=6cm,
        xtick={-1,0,1},
        xticklabels = {\small Negative interaction, \small No interaction, \small Positive interaction}
    }]
    \addplot [draw=none] coordinates {(0,0) (1,1)};
\end{axis}
\end{tikzpicture}
    \else
    \includegraphics[]{fig/\filename-figure\thefiguerNumber.pdf}
    \stepcounter{figuerNumber}
    \includegraphics[]{fig/\filename-figure\thefiguerNumber.pdf}
    \stepcounter{figuerNumber}
    \fi
    \caption{Feature importance and interactions for the movie review given in the introduction generated via \BII. \BII is able to pick up on the strong positive interaction between ``not'' and ``bad'' that leads to a positive prediction for this review. Note that the individual importance of ``not'' and ``bad'' is rather weak - their synergy is assigned importance.}
    \label{fig:intro_bii}
\end{figure}

\begin{figure}
	\centering
	\if\compileFigures1
    \begin{tikzpicture}
  \begin{axis}[width=\columnwidth,
    xbar,
    axis y line*=left,
	axis x line*=bottom,
    height=5cm,
    enlarge x limits=0.5,
    xlabel={Influence},
    symbolic y coords={every, isn't, bad, greatest, not},
    ytick=data,
    nodes near coords, nodes near coords align={horizontal},
    ]
\draw (axis cs:0,every) -- (axis cs:0,not);
\addplot[cellColor=539,forget plot,text  =black] coordinates {(0.21812132157699332,every)};
\addplot[cellColor=442,forget plot,text  =black] coordinates {(-0.3136755150315046,isn't)};
\addplot[cellColor=369,forget plot,text  =black] coordinates {(-0.7123636056530072,bad)};
\addplot[cellColor=679,forget plot,text  =black] coordinates {(0.9834961592054411,greatest)};
\addplot[cellColor=1000,text  =black] coordinates {(2.7398088625416555,not)};

  \end{axis}
\end{tikzpicture}
    \else
    \includegraphics[]{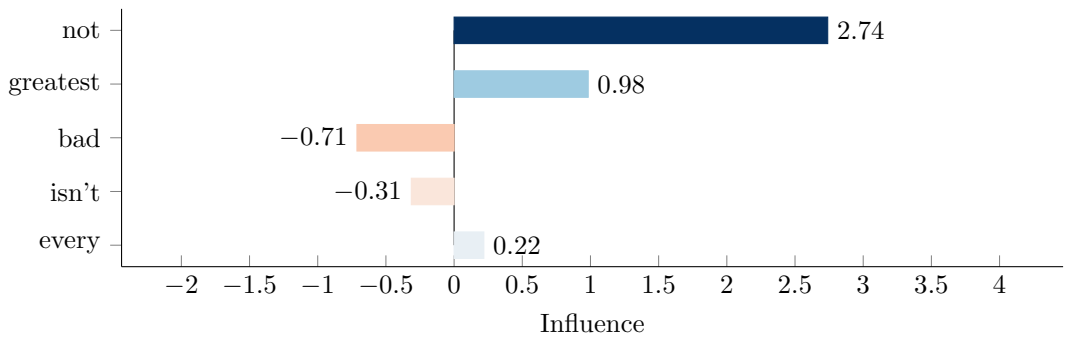}
    \stepcounter{figuerNumber}
    \fi
    \caption{\LIME feature importance for the movie review given in Section \ref{sec:intro}. The measure assigns significant weight to the word ``not'', but little weight to ``bad''.}
    \label{fig:intro_lime}
\end{figure}

In Figure~\ref{fig:treeExplanations_2}, we plot the explanation generated by \BII and \LIME for individuals with $\MS=$yes, $\CG\in\{5095,7073\}$ and $\EL < 12$ from the adult dataset for the simple tree classifier (see Figure~\ref{fig:tree}). \LIME suggests that $\CG$ is the most important feature and the other two features have minor importance. However, for any coalition of features (activated features are set to their POI value, and the others to their baseline value: $\MS = \texttt{No}$ and $\CG=0$), changing only $\MS$ or $\CG$ to their POI value does not significantly change the prediction value; but, changing both of them concurrently to their POI value results in a significant increase in prediction probability (see Table~\ref{tab:importance_of}). $\BII$ successfully captures these insights in Figure~\ref{fig:treeExplanations_2}, and suggests that the high positive interaction between $\CG$ and $\MS$ is responsible for the positive label.
When considering \emph{feature interaction}, there is high positive interaction between \CG and \MS, an insight that linear explanations do not capture, as seen in Table~\ref{tab:importance_of}. 
 
Figure~\ref{fig:treeExplanations} summarizes the explanations generated by \BII, \SII, \LIME and \QII for a married woman with a $\CG$ (CG) of 5178 with $\EL = 13$ in the adult test data for the simple tree classifier (Figure~\ref{fig:tree}).  Column (a) is the additive expansion of the feature-importance versions of these measures and Column (b) shows the methods that explicitly consider interactions.  
Both \BII and \SII indicate a negative synergy between $\EL$ and $\CG$, whereas all methods indicate positive individual importance. Hence, additive expansions fail to capture negative interactions, even for relatively simple models.

\begin{figure}
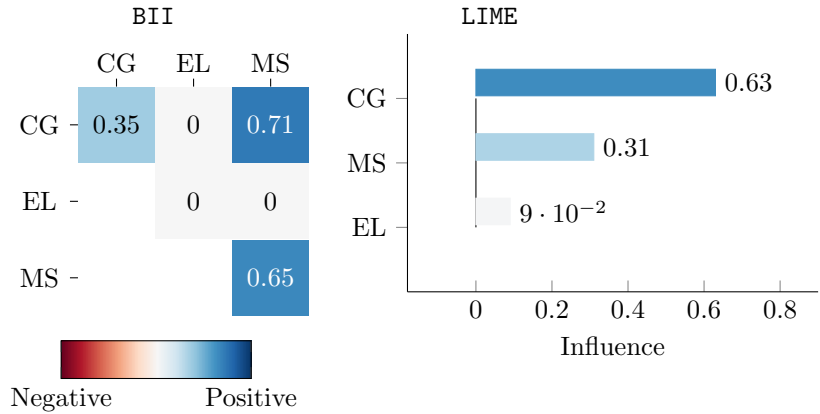

	\centering
	\pgfplotsset{colormap={RdBu}{ rgb255=(103,0,31) rgb255=(176,23,42) rgb255=(214,96,77) rgb255=(243,163,128) rgb255=(253,219,199) rgb255=(246,246,246) rgb255=(209,229,240) rgb255=(144,196,221) rgb255=(67,147,195) rgb255=(32,100,170) rgb255=(5,48,97) }}
\def\labels{{"CG","EL","MS"}}
\tikzset{cellColor/.style={color of colormap={#1},fill=.}}
\tikzset{cell/.style={minimum width=1cm,minimum height=1cm},
         hcell/.style={minimum width=1cm,minimum height=0.6cm},
         vcell/.style={minimum width=1cm,minimum height=1cm},
         ccell/.style={minimum width=1cm,minimum height=0.6cm},}

\newcommand{\HeatmapNum}[0]{
\pgfkeys{/pgf/number format/.cd,fixed,precision=2}
\node[ccell] (c-1-1) {};
\foreach \i [count=\iold from -1] in {0,1,2}{
        \node[cell,below=0cm of c\iold\iold] (c\i\iold) {};
    \pgfmathparse{\labels[\i]}\let\label\pgfmathresult
    \node[vcell,below=0cm of c\iold-1] (c\i-1) {\small \label};
    \draw (c\i-1.east) -- ([xshift=-3pt]c\i-1.east);
    \node[hcell,right=0cm of c-1\iold] (c-1\i) {\small \label};
    \draw (c-1\i.south) -- ([yshift=3pt]c-1\i.south);
    \foreach \j [count=\jold from \iold] in {\i,...,2}{
        \pgfmathparse{\values[\i][\j]}\let\value\pgfmathresult
        \pgfmathparse{int((\value+1)/(2*1)*1000)}\let\bkvalue\pgfmathresult
        \ifnum \bkvalue < 750
            \ifnum \bkvalue > 250
                \node[cell,cellColor=\bkvalue,right=0cm of c\i\jold,text=black] (c\i\j) {\small \pgfmathprintnumber{\value}};
            \else
                \node[cell,cellColor=\bkvalue,right=0cm of c\i\jold,text=white] (c\i\j) {\small \pgfmathprintnumber{\value}};
            \fi
        \else
            \node[cell,cellColor=\bkvalue,right=0cm of c\i\jold,text=white] (c\i\j) {\small \pgfmathprintnumber{\value}};
        \fi
    }
    }
}

\begin{tabular}{m{4cm} m{4cm}}
  \multicolumn{1}{c}{\BII}  &  \multicolumn{1}{c}{\LIME}  \\

\if\compileFigures1
\begin{tikzpicture}
\def\values{{
{0.35, 0.00, 0.71, },
{0.00, 0.00, 0.00, },
{0.00, 0.00, 0.65, },
}}
\HeatmapNum
\end{tikzpicture}
\begin{tikzpicture}
	\begin{axis}[
		hide axis,
		scale only axis,
		height=0pt,
		width=0pt,
		colorbar horizontal,
		point meta min =-1,
		point meta max =1,
		colorbar style={
			width=2.5cm,
			xtick={-1,1},
			xticklabels = {Negative , Positive }
		}]
		\addplot [draw=none] coordinates {(0,0) (1,1)};
	\end{axis}
\end{tikzpicture}
\else
\includegraphics[]{fig/\filename-figure\thefiguerNumber.pdf}
\stepcounter{figuerNumber}
\includegraphics[]{fig/\filename-figure\thefiguerNumber.pdf}
\stepcounter{figuerNumber}
\fi
&
\if\compileFigures1
\begin{tikzpicture}
	\begin{axis}[width=.5\columnwidth,height=4cm,
		xbar,
		axis y line*=left,
		axis x line*=bottom,
		height=5cm,
		enlarge x limits=0.5,
		enlarge y limits=0.5,
		xlabel={Influence},
		symbolic y coords={EL, MS, CG},
		ytick=data,
		nodes near coords, nodes near coords align={horizontal},
		]
		\draw (axis cs:0,EL) -- (axis cs:0,CG);
		\addplot[cellColor=815,forget plot,text  =black] coordinates {(0.63,CG)};
		\addplot[cellColor=655,forget plot,text  =black] coordinates {(0.31,MS)};
		\addplot[cellColor=505,forget plot,text  =black] coordinates {(0.09,EL)};	
	\end{axis}
\end{tikzpicture}
\else
\includegraphics[]{fig/\filename-figure\thefiguerNumber.pdf}
\stepcounter{figuerNumber}
\fi
     
\end{tabular}
\\
	\caption{The explanation generated by \BII and \LIME for another point in the dataset.\LIME fails to capture interactions.}
	\label{fig:treeExplanations_2}
\end{figure}

\begin{figure*}
    \centering
    \pgfplotsset{colormap={RdBu}{ rgb255=(103,0,31) rgb255=(176,23,42) rgb255=(214,96,77) rgb255=(243,163,128) rgb255=(253,219,199) rgb255=(246,246,246) rgb255=(209,229,240) rgb255=(144,196,221) rgb255=(67,147,195) rgb255=(32,100,170) rgb255=(5,48,97) }}
\def\labels{{"CG","EL","MS"}}
\tikzset{cellColor/.style={color of colormap={#1},fill=.}}
\tikzset{cell/.style={minimum width=0.95cm,minimum height=0.95cm},
         hcell/.style={minimum width=0.95cm,minimum height=0.6cm},
         vcell/.style={minimum width=0.95cm,minimum height=0.95cm},
         ccell/.style={minimum width=0.95cm,minimum height=0.6cm},}
\newcommand{\heatmapTextColor}[1]
{
\ifnum\numexpr#1>500
"white"
\else
"black"
\fi
}

\newcommand{\DHeatmap}[0]{
\pgfkeys{/pgf/number format/.cd,fixed,precision=2}
\node[ccell] (c-1-1)  {};
\foreach \i [count=\iold from -1] in {0,1,2}{
    \pgfmathparse{\labels[\i]}\let\label\pgfmathresult
    \node[cell,below=0cm of c\iold\iold] (c\i\iold) {};
    \node[vcell,below=0cm of c\iold-1] (c\i-1) {\label};
    \draw (c\i-1.east) -- ([xshift=-3pt]c\i-1.east);
    \node[hcell,right=0cm of c-1\iold] (c-1\i) {\label};
    \draw (c-1\i.south) -- ([yshift=3pt]c-1\i.south);
    \foreach \j [count=\jold from \iold] in {\i}{
        \pgfmathparse{\values[\i][\j]}\let\value\pgfmathresult
        \pgfmathparse{int((\value+1)/(2*1)*1000)}\let\bkvalue\pgfmathresult
        \ifnum \bkvalue < 750
            \ifnum \bkvalue > 250
                \node[cell,cellColor=\bkvalue,right=0cm of c\i\jold,text=black] (c\i\j) {\pgfmathprintnumber{\value}};
            \else
                \node[cell,cellColor=\bkvalue,right=0cm of c\i\jold,text=white] (c\i\j) {\pgfmathprintnumber{\value}};
            \fi
        \else
            \node[cell,cellColor=\bkvalue,right=0cm of c\i\jold,text=white] (c\i\j) {\pgfmathprintnumber{\value}};
        \fi
    }
    }
}

\newcommand{\Heatmap}[0]{
\pgfkeys{/pgf/number format/.cd,fixed,precision=2}
\node[ccell] (c-1-1) {};
\foreach \i [count=\iold from -1] in {0,1,2}{
        \node[cell,below=0cm of c\iold\iold] (c\i\iold) {};
    \pgfmathparse{\labels[\i]}\let\label\pgfmathresult
    \node[vcell,below=0cm of c\iold-1] (c\i-1) {\label};
    \draw (c\i-1.east) -- ([xshift=-3pt]c\i-1.east);
    \node[hcell,right=0cm of c-1\iold] (c-1\i) {\label};
    \draw (c-1\i.south) -- ([yshift=3pt]c-1\i.south);
    \foreach \j [count=\jold from \iold] in {\i,...,2}{
        \pgfmathparse{\values[\i][\j]}\let\value\pgfmathresult
        \pgfmathparse{int((\value+1)/(2*1)*1000)}\let\bkvalue\pgfmathresult
        \ifnum \bkvalue < 750
            \ifnum \bkvalue > 250
                \node[,cell,cellColor=\bkvalue,right=0cm of c\i\jold,text=black] (c\i\j) {\pgfmathprintnumber{\value}};
            \else
                \node[cell,cellColor=\bkvalue,right=0cm of c\i\jold,text=white] (c\i\j) {\pgfmathprintnumber{\value}};
            \fi
        \else
            \node[cell,cellColor=\bkvalue,right=0cm of c\i\jold,text=white] (c\i\j) {\pgfmathprintnumber{\value}};
        \fi
    }
    }
}

\begin{tabular}{l >{\centering\arraybackslash}m{3.6cm}>{\centering\arraybackslash}m{3.6cm}>{\centering\arraybackslash}m{3.6cm}}
\toprule
 &\textbf{(a) Additive expansion}&\textbf{(b) Actual interactions}&\textbf{(c) After monotone shift} \\ 
 \midrule
\textbf{Banzhaf} &
\if\compileFigures1
\begin{tikzpicture}
\def\values{{
{0.28, 0.46, 1.00, },
{0.00, 0.18, 0.90, },
{0.00, 0.00, 0.72, },
}}
\Heatmap
\end{tikzpicture}
\else
\includegraphics[]{fig/\filename-figure\thefiguerNumber.pdf}
\stepcounter{figuerNumber}
\fi
&
\if\compileFigures1
\begin{tikzpicture}
\def\values{{
{0.28, -0.20, 0.56, },
{0.00, 0.18, 0.07, },
{0.00, 0.00, 0.72, },
}}
\Heatmap
\end{tikzpicture}
\else
\includegraphics[]{fig/\filename-figure\thefiguerNumber.pdf}
\stepcounter{figuerNumber}
\fi
&
\if\compileFigures1
\begin{tikzpicture}
\def\values{{
{0.44, 0.12, 0.88, },
{0.00, 0.02, -0.2, },
{0.00, 0.00, 0.57, },
}}
\Heatmap
\end{tikzpicture}
\else
\includegraphics[]{fig/\filename-figure\thefiguerNumber.pdf}
\stepcounter{figuerNumber}
\fi
\\
\textbf{ Shapley }  &
\if\compileFigures1
\begin{tikzpicture}
\def\values{{
{0.22, 0.40, 1.00, },
{0.00, 0.18, 0.96, },
{0.00, 0.00, 0.78, },
}}
\Heatmap
\end{tikzpicture}
\else
\includegraphics[]{fig/\filename-figure\thefiguerNumber.pdf}
\stepcounter{figuerNumber}
\fi
&
\if\compileFigures1
\begin{tikzpicture}
\def\values{{
{0.22, -0.2, 0.63, },
{0.00, 0.18, 0.10, },
{0.00, 0.00, 0.78, },
}}
\Heatmap
\end{tikzpicture}
\else
\includegraphics[]{fig/\filename-figure\thefiguerNumber.pdf}
\stepcounter{figuerNumber}
\fi
&
\if\compileFigures1
\begin{tikzpicture}
\def\values{{
{0.43, 0.13, 0.90, },
{0.00, 0.04, -0.3, },
{0.00, 0.00, 0.60, },
}}
\Heatmap
\end{tikzpicture}
\else
\includegraphics[]{fig/\filename-figure\thefiguerNumber.pdf}
\stepcounter{figuerNumber}
\fi
\\
\textbf{Qii }&
\if\compileFigures1
\begin{tikzpicture}
\def\values{{
{0.28, 0.29, 1.00, },
{0.00, 0.01, 0.74, },
{0.00, 0.00, 0.72, },
}}
\Heatmap
\end{tikzpicture}
\else
\includegraphics[]{fig/\filename-figure\thefiguerNumber.pdf}
\stepcounter{figuerNumber}
\fi
&
\if\compileFigures1
\begin{tikzpicture}
\def\values{{
{0.28, 0.60, 0.72, },
{0.00, 0.01, 0.83, },
{0.00, 0.00, 0.72, },
}}
\Heatmap
\end{tikzpicture}
\else
\includegraphics[]{fig/\filename-figure\thefiguerNumber.pdf}
\stepcounter{figuerNumber}
\fi
&
\if\compileFigures1
\begin{tikzpicture}
\def\values{{
{0.76, 0.60, 0.72, },
{0.00, 0.01, 0.83, },
{0.00, 0.00, 0.72, },
}}
\Heatmap
\end{tikzpicture}
\else
\includegraphics[]{fig/\filename-figure\thefiguerNumber.pdf}
\stepcounter{figuerNumber}
\fi
\\
%
%
& \multicolumn{3}{c}{
\if\compileFigures1
\begin{tikzpicture}
\begin{axis}[
    hide axis,
    scale only axis,
    height=0pt,
    width=0pt,
    colorbar horizontal,
    point meta min =-1,
    point meta max =1,
    colorbar style={
        width=7cm,
        xtick={-1,0,1},
        xticklabels = {Negative interaction, No interaction, Positive interaction}
    }]
    \addplot [draw=none] coordinates {(0,0) (1,1)};
\end{axis}
\end{tikzpicture}
\else
\includegraphics[]{fig/\filename-figure\thefiguerNumber.pdf}
\stepcounter{figuerNumber}
\fi}\\
\bottomrule
\end{tabular}
    \caption{\small The explanations for a married woman with a $\CG$ (CG) of 5178 with $\EL = 13$ (i.e. Masters) in the adult test data for the above tree. (a) Linear explanations highlight the importance of $\MS$ (MS) and $\CG$ (CG), they fail to show how these to features interact. (b) The positive interaction between $\CG$ (CG) and $\MS$ (MS) highlights that both features a needed for a positive outcome. (c) After the tree gets modified to strengthen the interaction, the explanation proposed in Set-QII does not change, failing monotonicity. }
    \label{fig:treeExplanations}
\end{figure*}




\begin{table*}
	\small
    \caption{\small This example shows the joint importance of the interactions between $\CG$ and $\MS$ in the coalition of $\{\Age,\EL\}$, and the monotonicity property in the tree $T$ (Figure~\ref{fig:tree}). The prediction probability becomes $1.00$ when $\MS$ and $\CG$ both join the coalition $\{\Age,\EL\}$; however, when the $\MS$ or $\CG$ join individually in the coalition $\{\Age, \EL\}$, the prediction probability becomes $0.67$ and $0.15$ respectively. This shows that the features have synergistic effect, which is not captured by linear explanations (this holds for other coalitions, not just $\{\Age,\EL\}$). 
    }
	\setlength{\tabcolsep}{0pt}
    \centering
    \label{tab:importance_of}
    \if\compileFigures1
    \tikzexternaldisable
    \fi
    \begin{tabular}{ l c c c c c c}
        \toprule
         \textbf{Features} &\textbf{POI} & \textbf{Base} & \multicolumn{4}{c}{\textbf{Coalition}}\\
          & & &\textbf{\small \{A,EL\}} & \textbf{\small \{A,EL,MS\}} &\textbf{\small \{A,EL,CG\}} & \textbf{\small \{A,EL,MS,CG\}}  \\
         \midrule 
         \wc{$\MS$ (MS)} & \wc{Yes} &    \bc{No} &      \bc{No} & \wc{Yes} &     \bc{No} & \wc{Yes}\\
        \wc{$ \CG$ (CG)} & \wc{6021} &     \bc{0} &      \bc{0} &     \bc{0} & \wc{6021} & \wc{6021}\\
        \wc{$ \Age$ (A) }& \wc{32} &      \bc{37} & \wc{32} & \wc{32} &\wc{32} & \wc{32}\\
        \wc{$ \EL$ (EL)} & \wc{13} &    \bc{10} & \wc{13} & \wc{13} & \wc{13} & \wc{13}\\
        \midrule
        \textbf{Prediction} &1.00 & 0.03 & 0.15 & 0.67 & 0.15 & 1.00\\
        \bottomrule
    \end{tabular}
\end{table*}

\subsection{Monotonicity Matters}\label{sec:monotonicity}
Having established the importance of measuring feature synergy, let us turn to evaluating the importance of our proposed axioms. To do so, we consider a modified version of the tree depicted in Figure~\ref{fig:tree}. The output in the second-to-last leaf is changed from 0.67 to 0.01, i.e. married people with high $\EL$ but low $\CG$ are now classified as not having a high paying job. This change does not affect the prediction of the two points discussed above. 
However, this increases the interaction between $\CG$ and $\MS$: having a lower $\CG$ (or being unmarried) both lead to a negative outcome. More formally, we denote the cooperative game induced by the original tree by $f$ and the one induced by the modified tree by $f'$. Under $f'$, the joint marginal gains increase: 
\begin{align*}
  m_{\text{MS,CG}}( &\{ \Age,\EL \},f') >\\
  &m_{\text{MS,CG}}(\{\Age,\EL\},f),  
\end{align*}
and $m_{\text{MS,CG}}(T,f') = m_{\text{MS,CG}}(T,f)$ for any other coalition $T$. According to the \emph{monotonicity} property, the interaction index value among $\{\Age,\EL\}$ should be higher under $f'$ than $f$.  
This change is reflected in Shapley interactions and \BII as both measures satisfy monotonicity (see Figure~\ref{fig:treeExplanations}, column (b) for actual interaction and column (c) for interaction after the modification). Set-QII, however, is unchanged; it fails to satisfy monotonicity, and even for this small example cannot distinguish between the two trees. Similarly, the interaction between $\CG$ and $\EN$ should also increase which is again reflected by \SII and \BII; Set-QII fails to show this (see Figure~\ref{fig:treeExplanations}).


\begin{table*}
\small
    \centering
    \caption{\small Failure to satisfy generalized 2-efficiency leads to counterintuitive interactions after features are merged. We merge $\EL$ and $\EN$ (which are actually identical); Shapley interaction values with $\MS$ (left) and $\CG$ (right) before and after the features are merged tend to overestimate the expected merged value (as compared to the expected value under generalised 2-efficiency).}
    \label{tab:featureMerging}
    \setlength{\tabcolsep}{0pt}
    \if\compileFigures1
    \tikzexternaldisable
    \fi
    \begin{tabular}{l c c c c c p{4pt}  c c c c c }
        \toprule
        & \multicolumn{5}{c}{Interaction with \MS}& & \multicolumn{5}{c}{Interaction with \CG}\\
        \cmidrule{2-6} \cmidrule{8-12}
                            & Pt. 1 & Pt. 2 & Pt. 3 & Pt. 4 & Pt. 5  && Pt. 1 & Pt. 2 & Pt. 3 & Pt. 4 & Pt. 5\\
    \midrule
    $\EL $ (EL)   &-.068 &-.032&-.028&  -.066 &  -.044& & -.017 &0 & .016	 & .015	 & -.002  \\
    $\EN $ (EN)                   &.05 &.123&.118&  .041 &  .17& & .016 &-.007 & -.009	 & -.006	 & -.011  \\
    \whitecell{$[\text{EL} ;\text{EN}]$} &\colorcell{-.077}{\MSscale}   &\colorcell{.161}{\MSscale}  &\colorcell{.154}{\MSscale}  &  \colorcell{ -.077}{\MSscale} &  \colorcell{.23}{\MSscale}&& \colorcell{.013}{\CGscale} &\colorcell{-.013}{\CGscale} & \colorcell{.017}{\CGscale}	 & \colorcell{.019}{\CGscale}	 & \colorcell{.014}{\CGscale}  \\
    \whitecell{Expected [EL;EN]}    &\colorcell{-.018}{\MSscale}   &\colorcell{.091}{\MSscale}  &\colorcell{.090}{\MSscale}  &  \colorcell{-.025}{\MSscale} &  \colorcell{.126}{\MSscale}&& \colorcell{-.001}{\CGscale} &\colorcell{-.007}{\CGscale} & \colorcell{.007}{\CGscale}	 & \colorcell{.009}{\CGscale}	 & \colorcell{.009}{\CGscale}  \\
    \midrule
     \whitecell{\textbf{Difference}}&\colorcell{-.059}{\MSscale}   &\colorcell{.070}{\MSscale}  &\colorcell{.064}{\MSscale}  &  \colorcell{-.052}{\MSscale} &  \colorcell{-.103}{\MSscale}&& \colorcell{.014}{\CGscale} &\colorcell{-.006}{\CGscale} & \colorcell{.010}{\CGscale}	 & \colorcell{.010}{\CGscale}	 & \colorcell{.005}{\CGscale}  \\
    \bottomrule
    \end{tabular}
\if\compileFigures1
\tikzexternalenable
\begin{tikzpicture}
\begin{axis}[
    hide axis,
    scale only axis,
    height=0pt,
    width=0pt,
    colorbar horizontal,
    point meta min =-1,
    point meta max =1,
    colorbar style={
        width=7cm,
        xtick={-1,0,1},
        xticklabels = {Negative interaction, No interaction, Positive interaction}
    }]
    \addplot [draw=none] coordinates {(0,0) (1,1)};
\end{axis}
\end{tikzpicture}
\else
\includegraphics[]{fig/\filename-figure\thefiguerNumber.pdf}
\stepcounter{figuerNumber}
\fi
\end{table*}

\subsection{The Effects of 2-Efficiency}
Table~\ref{tab:featureMerging} highlights how failure to satisfy generalized 2-efficiency can lead to counterintuitive interactions when features are merged. 
We calculated the \SII values for $\EL$ and $\EN$ for a random forest with $\MS$ (left) and $\CG$ (right) on some sampled points; we recalculated them after merging the features to $[\EL;\EN]$. These features can be naturally merged as they are in fact identical: $\EN$ is just a numerical representation of $\EL$. 

In some instances, the interaction values of $\EL$ and $\EN$ with other features differ in sign (i.e. one has positive interaction and another has a negative interaction). In this case, under 2-efficiency, the interaction value of the merged feature $[\EL;\EN]$ should lie somewhere between the two.
In Table~\ref{tab:featureMerging}, we show some of the points where the Shapley interaction of $\MS$ and $\CG$ with merged feature deviates from the expected interactions. For example, for Point 2, the \SII of $\EL$ with $\MS$ is $-0.32$ and $\EN$ with $\MS$ is $0.123$, therefore the interaction after merging $\EN$ and $\EL$ should be less than $0.123$; however, the \SII interaction value between $\MS$ and the merged feature is $0.161$ --- more than $0.123$! In other words, \SII placed far too much influence on two identical merged features. We show additional examples in Table~\ref{tab:featureMerging}.
Since \BII satisfies generalized 2-efficiency, influence is preserved under merging and thus matches the expected outcome.
\newpage 
\section{Additional Example Explanations for BERT}\label{sec:BERTex}
\begin{figure}[!htb]
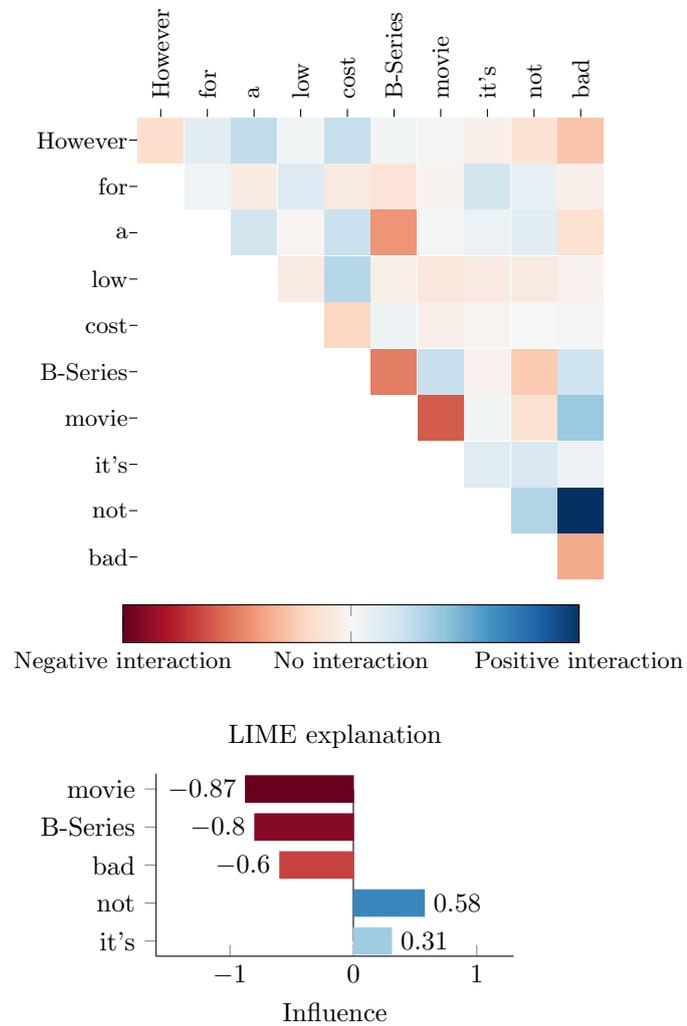

    \centering
    \begin{subfigure}[b]{0.5\textwidth}
    \if\compileFigures1
    \input{figures/heatmapDefinition.tex}
    \begin{tikzpicture}
\def\values{{
{-0.17, 0.11, 0.25, 0.03, 0.24, 0.03, -0.01, -0.05, -0.15, -0.28, },
{0.00, 0.04, -0.07, 0.13, -0.09, -0.13, -0.03, 0.19, 0.08, -0.05, },
{0.00, 0.00, 0.19, -0.02, 0.22, -0.44, 0.02, 0.06, 0.11, -0.16, },
{0.00, 0.00, 0.00, -0.08, 0.29, -0.06, -0.11, -0.08, -0.09, -0.04, },
{0.00, 0.00, 0.00, 0.00, -0.21, 0.05, -0.05, -0.02, 0.00, 0.01, },
{0.00, 0.00, 0.00, 0.00, 0.00, -0.51, 0.23, -0.04, -0.26, 0.21, },
{0.00, 0.00, 0.00, 0.00, 0.00, 0.00, -0.61, 0.03, -0.15, 0.37, },
{0.00, 0.00, 0.00, 0.00, 0.00, 0.00, 0.00, 0.12, 0.16, 0.06, },
{0.00, 0.00, 0.00, 0.00, 0.00, 0.00, 0.00, 0.00, 0.30, 1.00, },
{0.00, 0.00, 0.00, 0.00, 0.00, 0.00, 0.00, 0.00, 0.00, -0.37, },
}}
\def\labels{{"However", "for", "a", "low", "cost", "B-Series", "movie", "it's", "not", "bad",  }}
\def\width{.6}
\def\twidth{1.5}
\def\height{.6}
\Heatmap{9}{\small}
\end{tikzpicture} \\
\begin{tikzpicture}
\begin{axis}[
    hide axis,
    scale only axis,
    height=0pt,
    width=0pt,
    colorbar horizontal,
    point meta min =-1,
    point meta max =1,
    colorbar style={
        width=6cm,
        xtick={-1,0,1},
        xticklabels = {\small Negative interaction, \small No interaction, \small Positive interaction}
    }]
    \addplot [draw=none] coordinates {(0,0) (1,1)};
\end{axis}
\end{tikzpicture}
    \else
    \includegraphics[]{fig/\filename-figure\thefiguerNumber.pdf}
    \stepcounter{figuerNumber}
    \includegraphics[]{fig/\filename-figure\thefiguerNumber.pdf}
    \stepcounter{figuerNumber}
    \fi
	\end{subfigure}\\
    \begin{subfigure}[b]{0.45\textwidth}
    	\if\compileFigures1
    	\begin{tikzpicture}
  \begin{axis}[width=\linewidth,
    xbar=1mm,
    axis y line*=left,
	axis x line*=bottom,
    height=4cm,
    enlarge x limits=0.5,
    xlabel={Influence},
    symbolic y coords={it's, not, bad, B-Series, movie, },
    ytick=data,
    nodes near coords, nodes near coords align={horizontal},
    title=LIME explanation
    ]
\draw (axis cs:0,it's) -- (axis cs:0,movie);
\addplot[cellColor=676,forget plot,text  =black] coordinates {(0.30891096367354764,it's)};
\addplot[cellColor=829,forget plot,text  =black] coordinates {(0.575090991342599,not)};
\addplot[cellColor=159,forget plot,text  =black] coordinates {(-0.5956265562633294,bad)};
\addplot[cellColor=42,forget plot,text  =black] coordinates {(-0.7986898933042558,B-Series)};
\addplot[cellColor=0,text  =black] coordinates {(-0.8737076201666895,movie)};
  \end{axis}
\end{tikzpicture}
    	\else
    	\includegraphics[]{fig/\filename-figure\thefiguerNumber.pdf}
    	\stepcounter{figuerNumber}
    	\fi
    \end{subfigure}
    \caption{ While most individual features have negative influence for both \BII and \LIME, \BII is able to pick up on the strong positive interaction between "not" and "bad" that leads to a positive prediction of this review.}
    \label{fig:bii_lime_283}
\end{figure}
\newpage

\begin{figure}[!htb]
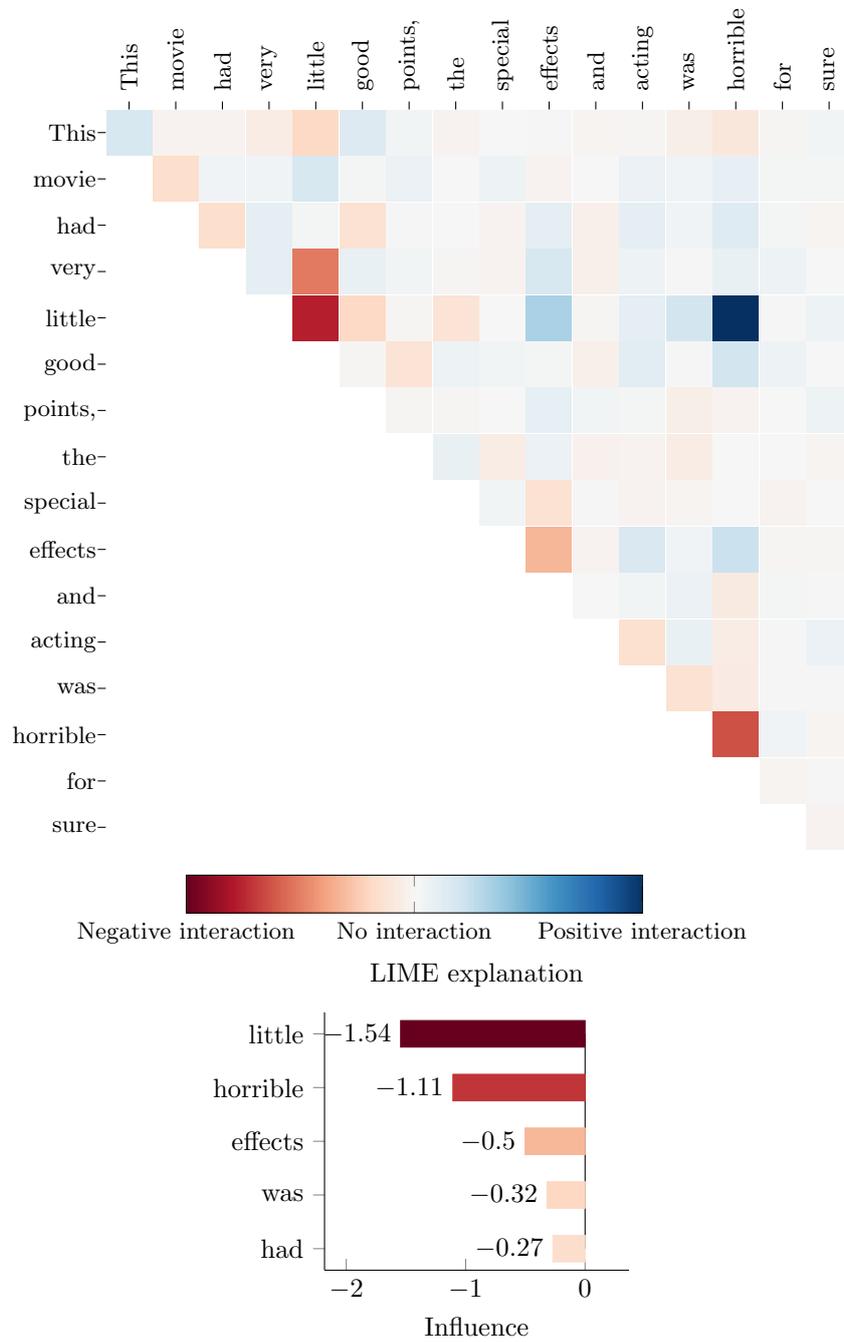

	\centering
	\begin{subfigure}[b]{\textwidth}
		\centering
		\if\compileFigures1
		\input{figures/heatmapDefinition.tex}
		\begin{tikzpicture}
\def\labels{{"This", "movie", "had", "very", "little", "good", "points,", "the", "special", "effects", "and", "acting", "was", "horrible", "for", "sure", }}
\def\values{{
{0.17, -0.03, -0.03, -0.07, -0.20, 0.14, 0.03, -0.03, 0.00, 0.01, -0.02, -0.01, -0.06, -0.11, -0.01, 0.03, },
{0.00, -0.17, 0.04, 0.04, 0.17, 0.02, 0.06, 0.00, 0.05, -0.03, 0.00, 0.06, 0.04, 0.08, 0.02, 0.02, },
{0.00, 0.00, -0.17, 0.10, 0.02, -0.15, 0.01, -0.00, -0.03, 0.10, -0.05, 0.10, 0.04, 0.13, 0.02, -0.02, },
{0.00, 0.00, 0.00, 0.10, -0.52, 0.07, 0.03, -0.01, -0.03, 0.17, -0.05, 0.05, 0.01, 0.07, 0.05, -0.00, },
{0.00, 0.00, 0.00, 0.00, -0.78, -0.20, -0.01, -0.14, -0.00, 0.32, -0.01, 0.10, 0.19, 1.00, 0.01, 0.05, },
{0.00, 0.00, 0.00, 0.00, 0.00, -0.01, -0.14, 0.05, 0.03, 0.02, -0.05, 0.11, 0.01, 0.19, 0.05, 0.00, },
{0.00, 0.00, 0.00, 0.00, 0.00, 0.00, -0.01, -0.01, -0.00, 0.09, 0.03, 0.02, -0.06, -0.03, -0.00, 0.05, },
{0.00, 0.00, 0.00, 0.00, 0.00, 0.00, 0.00, 0.07, -0.07, 0.06, -0.04, -0.03, -0.07, 0.00, -0.00, -0.02, },
{0.00, 0.00, 0.00, 0.00, 0.00, 0.00, 0.00, 0.00, 0.03, -0.15, 0.01, -0.03, -0.02, -0.00, -0.03, -0.00, },
{0.00, 0.00, 0.00, 0.00, 0.00, 0.00, 0.00, 0.00, 0.00, -0.33, -0.03, 0.16, 0.04, 0.22, -0.01, -0.01, },
{0.00, 0.00, 0.00, 0.00, 0.00, 0.00, 0.00, 0.00, 0.00, 0.00, 0.00, 0.03, 0.06, -0.09, 0.02, 0.01, },
{0.00, 0.00, 0.00, 0.00, 0.00, 0.00, 0.00, 0.00, 0.00, 0.00, 0.00, -0.16, 0.07, -0.07, 0.01, 0.06, },
{0.00, 0.00, 0.00, 0.00, 0.00, 0.00, 0.00, 0.00, 0.00, 0.00, 0.00, 0.00, -0.15, -0.08, 0.01, 0.01, },
{0.00, 0.00, 0.00, 0.00, 0.00, 0.00, 0.00, 0.00, 0.00, 0.00, 0.00, 0.00, 0.00, -0.64, 0.04, -0.02, },
{0.00, 0.00, 0.00, 0.00, 0.00, 0.00, 0.00, 0.00, 0.00, 0.00, 0.00, 0.00, 0.00, 0.00, -0.02, 0.01, },
{0.0, 0.00, 0.00, 0.00, 0.00, 0.00, 0.00, 0.00, 0.00, 0.00, 0.00, 0.00, 0.00, 0.00, 0.00, -0.03, },
}}

\def\width{.6}
\def\twidth{1.4}
\def\height{.6}
\Heatmap{15}{\small}
\end{tikzpicture}\\
\begin{tikzpicture}
\begin{axis}[
    hide axis,
    scale only axis,
    height=0pt,
    width=0pt,
    colorbar horizontal,
    point meta min =-1,
    point meta max =1,
    colorbar style={
        width=6cm,
        xtick={-1,0,1},
        xticklabels = {\small Negative interaction, \small No interaction, \small Positive interaction}
    }]
    \addplot [draw=none] coordinates {(0,0) (1,1)};
\end{axis}
\end{tikzpicture}
		\else
		\includegraphics[]{fig/\filename-figure\thefiguerNumber.pdf}
		\stepcounter{figuerNumber}
		\includegraphics[]{fig/\filename-figure\thefiguerNumber.pdf}
		\stepcounter{figuerNumber}
		\fi
	\end{subfigure}\\
	\begin{subfigure}[b]{0.4\textwidth}
		\if\compileFigures1
		\begin{tikzpicture}
  \begin{axis}[width=\linewidth,
    xbar,
    axis y line*=left,
	axis x line*=bottom,
    height=5cm,
    enlarge x limits=0.5,
    xlabel={Influence},
    symbolic y coords={had, was, effects, horrible, little,   },
    ytick=data,
    nodes near coords, nodes near coords align={horizontal},
    title=LIME explanation
    ]
\draw (axis cs:0,little) -- (axis cs:0,had);
\addplot[cellColor=412,forget plot,text  =black] coordinates {(-0.26943514077122427,had)};
\addplot[cellColor=396,forget plot,text  =black] coordinates {(-0.3188024908706511,was)};
\addplot[cellColor=337,forget plot,text  =black] coordinates {(-0.5019748598691895,effects)};
\addplot[cellColor=141,forget plot,text  =black] coordinates {(-1.106545264946068,horrible)};
\addplot[cellColor=0,text  =black] coordinates {(-1.5445578475025283,little)};
  \end{axis}
\end{tikzpicture}
		\else
		\includegraphics[]{fig/\filename-figure\thefiguerNumber.pdf}
		\stepcounter{figuerNumber}
		\fi
	\end{subfigure}
	\caption{ While most individual features have negative influence suggested by \BII and \LIME, \BII is able to pick up on the strong positive interaction between "little" and "horrible" that leads to less confidence on the negative prediction of the review.}
	\label{fig:bii_lime_322}
\end{figure}
\newpage
\begin{figure}[!htb]
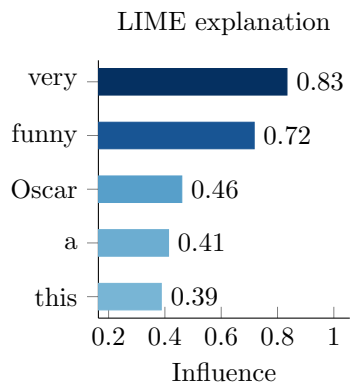

	\centering
	\begin{subfigure}[b]{0.6\textwidth}
		\centering
		\if\compileFigures1
		\input{figures/heatmapDefinition.tex}
		\begin{tikzpicture}
\def\labels{{"It", "may", "not", "be", "Oscar", "material,", "however", "this", "was", "a", "very", "funny", "film", }}
\def\values{{
{0.32, -0.03, 0.23, -0.01, 0.05, 0.09, 0.11, -0.07, -0.06, -0.11, -0.15, -0.16, -0.14, },
{0.00, 0.18, 0.28, 0.04, -0.08, 0.10, 0.10, -0.04, -0.06, -0.07, 0.01, -0.04, -0.04, },
{0.00, 0.00, -0.15, 0.08, 0.38, 0.94, 1.00, 0.28, 0.35, 0.07, -0.21, -0.42, 0.01, },
{0.00, 0.00, 0.00, 0.06, -0.19, 0.01, 0.12, 0.03, -0.00, 0.03, 0.03, 0.03, 0.02, },
{0.00, 0.00, 0.00, 0.00, 0.27, -0.39, 0.23, 0.12, 0.01, -0.12, -0.02, 0.29, 0.14, },
{0.00, 0.00, 0.00, 0.00, 0.00, 0.05, -0.33, 0.19, -0.25, 0.10, 0.22, 0.86, 0.22, },
{0.00, 0.00, 0.00, 0.00, 0.00, 0.00, -0.19, -0.02, -0.24, 0.16, 0.21, 0.41, 0.15, },
{0.00, 0.00, 0.00, 0.00, 0.00, 0.00, 0.00, 0.15, -0.05, -0.08, 0.02, 0.23, -0.06, },
{0.00, 0.00, 0.00, 0.00, 0.00, 0.00, 0.00, 0.00, 0.07, 0.03, 0.23, 0.06, -0.08, },
{0.00, 0.00, 0.00, 0.00, 0.00, 0.00, 0.00, 0.00, 0.00, 0.32, 0.06, -0.22, 0.02, },
{0.00, 0.00, 0.00, 0.00, 0.00, 0.00, 0.00, 0.00, 0.00, 0.00, 0.51, -0.52, 0.19, },
{0.00, 0.00, 0.00, 0.00, 0.00, 0.00, 0.00, 0.00, 0.00, 0.00, 0.00, 0.49, -0.00, },
{0.00, 0.00, 0.00, 0.00, 0.00, 0.00, 0.00, 0.00, 0.00, 0.00, 0.00, 0.00, 0.11, },
}}
\def\width{.6}
\def\twidth{1.6}
\def\height{.6}

\Heatmap{12}{\small}
\end{tikzpicture}\\
\begin{tikzpicture}
\begin{axis}[
    hide axis,
    scale only axis,
    height=0pt,
    width=0pt,
    colorbar horizontal,
    point meta min =-1,
    point meta max =1,
    colorbar style={
        width=6cm,
        xtick={-1,0,1},
        xticklabels = {\small Negative interaction, \small No interaction, \small Positive interaction}
    }]
    \addplot [draw=none] coordinates {(0,0) (1,1)};
\end{axis}
\end{tikzpicture}
		\else
		\includegraphics[]{fig/\filename-figure\thefiguerNumber.pdf}
		\stepcounter{figuerNumber}
		\includegraphics[]{fig/\filename-figure\thefiguerNumber.pdf}
		\stepcounter{figuerNumber}
		\fi
	\end{subfigure}\\
	\begin{subfigure}[b]{0.35\textwidth}
		\if\compileFigures1
		\begin{tikzpicture}
  \begin{axis}[width=\linewidth,
    xbar,
    axis y line*=left,
	axis x line*=bottom,
    height=5cm,
    enlarge x limits=0.5,
    xlabel={Influence},
    symbolic y coords={this, a, Oscar, funny, very,  },
    ytick=data,
    nodes near coords, nodes near coords align={horizontal},
    title=LIME explanation
    ]
\draw (axis cs:0,very) -- (axis cs:0,this);
\addplot[cellColor=732,forget plot,text  =black] coordinates {(0.3870368897445558,this)};
\addplot[cellColor=747,forget plot,text  =black] coordinates {(0.41287485945638214,a)};
\addplot[cellColor=775,forget plot,text  =black] coordinates {(0.45994187477166015,Oscar)};
\addplot[cellColor=930,forget plot,text  =black] coordinates {(0.7170381278691298,funny)};
\addplot[cellColor=1000,text  =black] coordinates {(0.8335028181840112,very)};

  \end{axis}
\end{tikzpicture}
		\else
		\includegraphics[]{fig/\filename-figure\thefiguerNumber.pdf}
		\stepcounter{figuerNumber}
		\fi
	\end{subfigure}
	\caption{While most individual features have positive influence for \BII and \LIME both, \BII is able to pick up some interesting interactions among words. For example the strong negative interaction between "not" and "funny" and "very" and "funny"  that leads to less confidence on the negative prediction of the review.}
	\label{fig:bii_lime_865}
\end{figure}
\newpage
\begin{figure}[!htb]
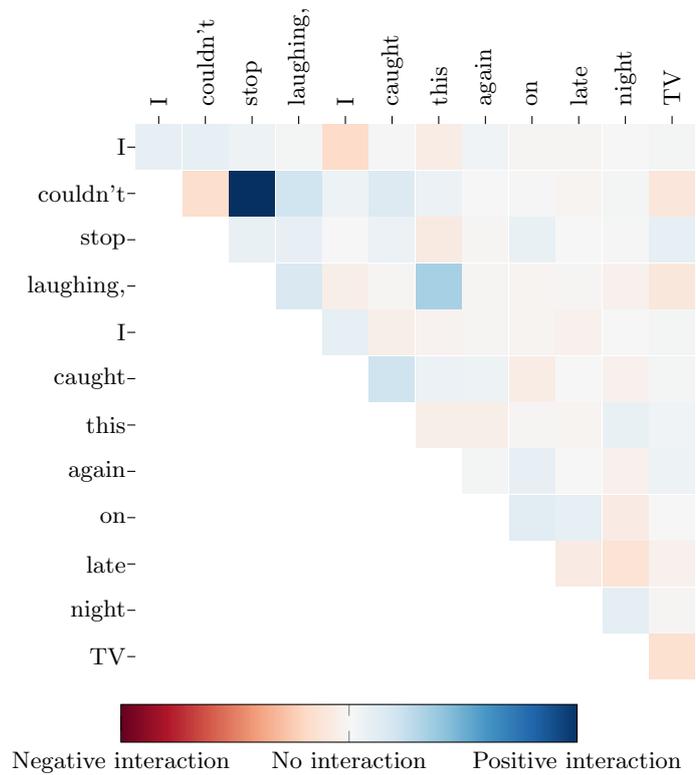
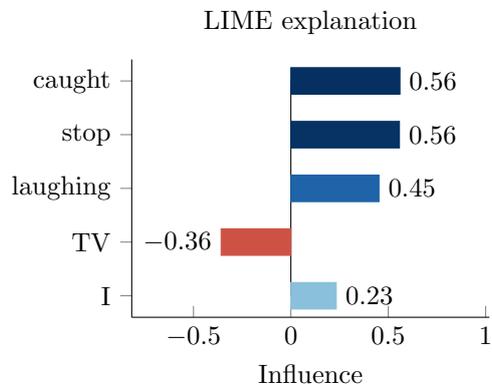

	\centering
	\begin{subfigure}[b]{0.5\textwidth}
		\if\compileFigures1
		\input{figures/heatmapDefinition.tex}
		\begin{tikzpicture}
\def\width{.6}
\def\twidth{1.5}
\def\height{.6}
\def\labels{{"I", "couldn't", "stop", "laughing,", "I", "caught", "this", "again", "on", "late", "night", "TV",  }}
\def\values{{
{0.08, 0.09, 0.05, 0.02, -0.19, 0.01, -0.07, 0.04, -0.01, -0.01, 0.00, 0.02, },
{0.00, -0.17, 1.00, 0.20, 0.05, 0.14, 0.06, -0.00, 0.01, -0.02, 0.02, -0.12, },
{0.00, 0.00, 0.07, 0.08, 0.00, 0.06, -0.09, -0.01, 0.07, -0.00, 0.01, 0.09, },
{0.00, 0.00, 0.00, 0.16, -0.06, -0.01, 0.33, -0.01, -0.02, -0.01, -0.04, -0.11, },
{0.00, 0.00, 0.00, 0.00, 0.09, -0.06, -0.03, -0.01, -0.02, -0.04, 0.00, 0.02, },
{0.00, 0.00, 0.00, 0.00, 0.00, 0.21, 0.06, 0.05, -0.07, -0.00, -0.04, 0.02, },
{0.00, 0.00, 0.00, 0.00, 0.00, 0.00, -0.06, -0.06, -0.01, -0.02, 0.07, 0.04, },
{0.00, 0.00, 0.00, 0.00, 0.00, 0.00, 0.00, 0.02, 0.08, -0.00, -0.04, 0.05, },
{0.00, 0.00, 0.00, 0.00, 0.00, 0.00, 0.00, 0.00, 0.11, 0.09, -0.08, -0.00, },
{0.00, 0.00, 0.00, 0.00, 0.00, 0.00, 0.00, 0.00, 0.00, -0.08, -0.14, -0.04, },
{0.00, 0.00, 0.00, 0.00, 0.00, 0.00, 0.00, 0.00, 0.00, 0.00, 0.10, -0.01, },
{0.00, 0.00, 0.00, 0.00, 0.00, 0.00, 0.00, 0.00, 0.00, 0.00, 0.00, -0.16, },
}}

\Heatmap{11}{\small}
\end{tikzpicture}\\
\begin{tikzpicture}
\begin{axis}[
    hide axis,
    scale only axis,
    height=0pt,
    width=0pt,
    colorbar horizontal,
    point meta min =-1,
    point meta max =1,
    colorbar style={
        width=6cm,
        xtick={-1,0,1},
        xticklabels = {\small Negative interaction, \small No interaction, \small Positive interaction}
    }]
    \addplot [draw=none] coordinates {(0,0) (1,1)};
\end{axis}
\end{tikzpicture}
		\else
		\includegraphics[]{fig/\filename-figure\thefiguerNumber.pdf}
		\stepcounter{figuerNumber}
		\includegraphics[]{fig/\filename-figure\thefiguerNumber.pdf}
		\stepcounter{figuerNumber}
		\fi
	\end{subfigure}\\
	\begin{subfigure}[b]{0.45\textwidth}
		\if\compileFigures1
		\begin{tikzpicture}
  \begin{axis}[width=\linewidth,
    xbar,
    axis y line*=left,
	axis x line*=bottom,
    height=5cm,
    enlarge x limits=0.5,
    xlabel={Influence},
    symbolic y coords={I, TV, laughing, stop, caught,  },
    ytick=data,
    nodes near coords, nodes near coords align={horizontal},
    title=LIME explanation
    ]
\draw (axis cs:0,caught) -- (axis cs:0,I);
\addplot[cellColor=708,forget plot,text  =black] coordinates {(0.2330334964232226,I)};
\addplot[cellColor=180,forget plot,text  =black] coordinates {(-0.3581529069140019,TV)};
\addplot[cellColor=903,forget plot,text  =black] coordinates {(0.45233863472701896,laughing)};
\addplot[cellColor=997,forget plot,text  =black] coordinates {(0.5575027293182119,stop)};
\addplot[cellColor=1000,text  =black] coordinates {(0.5599453495425041,caught)};

  \end{axis}
\end{tikzpicture}
		\else
		\includegraphics[]{fig/\filename-figure\thefiguerNumber.pdf}
		\stepcounter{figuerNumber}
		\fi
	\end{subfigure}
	\caption{\BII assigns almost negligible importance to individual word and picks up on the strong positive interaction between "couldn't" and "stop" that leads to a positive prediction of this review. However, \LIME or any other feature based model explanation fails to capture such interactions.}
	\label{fig:bii_lime_877}
\end{figure}
\newpage
\begin{figure}[!htb]
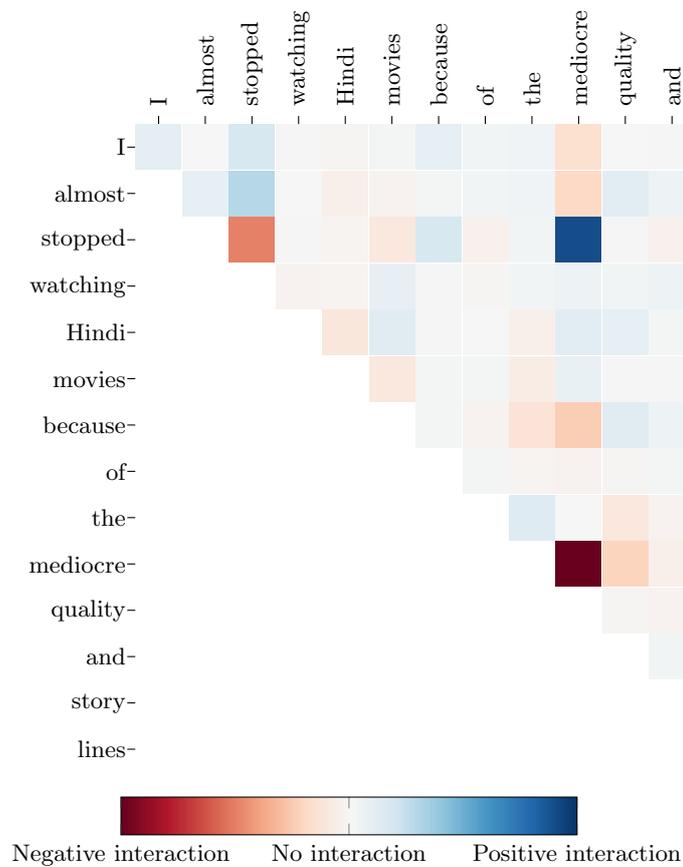
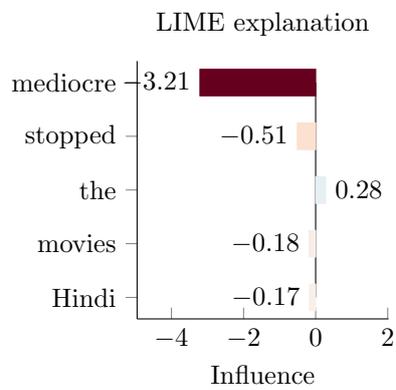

	\centering
	\begin{subfigure}[b]{0.6\textwidth}
		\if\compileFigures1
		\input{figures/heatmapDefinition.tex}
		\begin{tikzpicture}
\def\labels{{"I", "almost", "stopped", "watching", "Hindi", "movies", "because", "of", "the", "mediocre", "quality", "and", "story", "lines", }}
\def\values{{
{0.10, -0.00, 0.17, 0.01, -0.01, 0.02, 0.09, 0.03, 0.04, -0.16, -0.00, 0.01, 0.02, 0.01, },
{0.00, 0.09, 0.29, -0.00, -0.05, -0.03, 0.02, 0.03, 0.04, -0.20, 0.11, 0.05, 0.04, 0.04, },
{0.00, 0.00, -0.50, 0.01, -0.02, -0.10, 0.17, -0.04, 0.03, 0.89, 0.01, -0.04, 0.02, -0.03, },
{0.00, 0.00, 0.00, -0.03, -0.02, 0.08, 0.01, -0.01, 0.03, 0.05, 0.03, 0.05, 0.02, 0.03, },
{0.00, 0.00, 0.00, 0.00, -0.11, 0.12, 0.01, 0.00, -0.05, 0.11, 0.09, 0.02, 0.06, 0.02, },
{0.00, 0.00, 0.00, 0.00, 0.00, -0.10, 0.02, 0.02, -0.07, 0.07, 0.01, 0.01, 0.07, 0.00, },
{0.00, 0.00, 0.00, 0.00, 0.00, 0.00, 0.02, -0.03, -0.13, -0.25, 0.12, 0.05, 0.06, -0.03, },
{0.00, 0.00, 0.00, 0.00, 0.00, 0.00, 0.00, 0.02, -0.02, -0.03, -0.01, 0.02, 0.01, 0.00, },
{0.00, 0.00, 0.00, 0.00, 0.00, 0.00, 0.00, 0.00, 0.13, -0.00, -0.10, -0.03, -0.00, -0.01, },
{0.00, 0.00, 0.00, 0.00, 0.00, 0.00, 0.00, 0.00, 0.00, -1.00, -0.22, -0.05, -0.12, -0.01, },
{0.00, 0.00, 0.00, 0.00, 0.00, 0.00, 0.00, 0.00, 0.00, 0.00, -0.01, -0.03, 0.11, 0.03, },
{0.00, 0.00, 0.00, 0.00, 0.00, 0.00, 0.00, 0.00, 0.00, 0.00, 0.00, 0.03, 0.01, -0.02, },
{0.00, 0.00, 0.00, 0.00, 0.00, 0.00, 0.00, 0.00, 0.00, 0.00, 0.00, 0.00, 0.01, -0.11, },
{0.00, 0.00, 0.00, 0.00, 0.00, 0.00, 0.00, 0.00, 0.00, 0.00, 0.00, 0.00, 0.00, 0.00, },
}}
\def\width{.6}
\def\twidth{1.5}
\def\height{.6}
\Heatmap{13}{\small}
\end{tikzpicture}\\
\begin{tikzpicture}
\begin{axis}[
    hide axis,
    scale only axis,
    height=0pt,
    width=0pt,
    colorbar horizontal,
    point meta min =-1,
    point meta max =1,
    colorbar style={
        width=6cm,
        xtick={-1,0,1},
        xticklabels = {\small Negative interaction, \small No interaction, \small Positive interaction}
    }]
    \addplot [draw=none] coordinates {(0,0) (1,1)};
\end{axis}
\end{tikzpicture}
		\else
		\includegraphics[]{fig/\filename-figure\thefiguerNumber.pdf}
		\stepcounter{figuerNumber}
		\includegraphics[]{fig/\filename-figure\thefiguerNumber.pdf}
		\stepcounter{figuerNumber}
		\fi
	\end{subfigure}\\
	\begin{subfigure}[b]{0.35\textwidth}
		\if\compileFigures1
		\begin{tikzpicture}
  \begin{axis}[width=\linewidth,
    xbar,
    axis y line*=left,
	axis x line*=bottom,
    height=5cm,
    enlarge x limits=0.5,
    xlabel={Influence},
    symbolic y coords={Hindi, movies, the, stopped, mediocre,  },
    ytick=data,
    nodes near coords, nodes near coords align={horizontal},
    title=LIME explanation
    ]
\draw (axis cs:0,mediocre) -- (axis cs:0,Hindi);
\addplot[cellColor=473,forget plot,text  =black] coordinates {(-0.1726235292822736,Hindi)};
\addplot[cellColor=471,forget plot,text  =black] coordinates {(-0.18487898936242464,movies)};
\addplot[cellColor=543,forget plot,text  =black] coordinates {(0.2770612571089667,the)};
\addplot[cellColor=420,forget plot,text  =black] coordinates {(-0.5090384541568818,stopped)};
\addplot[cellColor=0,text  =black] coordinates {(-3.2104617305739906,mediocre)};

  \end{axis}
\end{tikzpicture}
		\else
		\includegraphics[]{fig/\filename-figure\thefiguerNumber.pdf}
		\stepcounter{figuerNumber}
		\fi
	\end{subfigure}
	\caption{While \BII and \LIME agrees on the feature importance level and suggests that "stopped" and "mediocre" has high negative influence. However, \BII points out that only one of the words "stopped" and "mediocre" has more negative influnce and they both together has positive interactions that leads to less confidence on the negative prediction of the review.}
	\label{fig:bii_lime_76}
\end{figure}

\newpage
\newpage
\section*{Contact}
\begin{contact}
	Neel Patel\\
	University of Southern California\\
	Los Angeles, USA\\
	\email{neelbpat@usc.edu}
\end{contact}

\begin{contact}
	Martin Strobel\\
	National University of Singapore\\
	Singapore, Singapore\\
	\email{mstrobel@comp.nus.edu.sg}
\end{contact}

\begin{contact}
	Yair Zick\\
	University of Massachusetts, Amherst\\
	Amherst, USA\\
	\email{yzick@umass.edu}
\end{contact}

\end{document}